\documentclass[journal]{IEEEtran} % If you use this template, delet the section References
\usepackage{amsthm}

\usepackage{cite}
\usepackage{amsmath,amssymb,amsfonts}
\usepackage{graphicx}
\usepackage{algorithm,algorithmic}
\usepackage{hyperref}
\hypersetup{hidelinks=true}
\usepackage{textcomp}
\usepackage{array}
\usepackage[caption=false,font=normalsize,labelfont=sf,textfont=sf]{subfig}
\usepackage{stfloats}
\usepackage{url}
\usepackage{verbatim}
\usepackage{multicol}
\usepackage{mathtools}
\usepackage{tikz}
\usepackage{pgfplots}
\usepackage{booktabs}       % professional-quality tables
\usepackage{multirow}
\usepackage{multicol}
\newtheorem{remark}{Remark}
\newtheorem{theorem}{Theorem}

\newtheorem{definition}{Definition}
\newtheorem{lemma}{Lemma}

\newtheorem{problem}{Problem}

\newcommand{\bbD}{\mathbb{D}}

\newcommand{\bbN}{\mathbb{N}}

\newcommand{\bbR}{\mathbb{R}}
\newcommand{\bbS}{\mathbb{S}}

\newcommand{\bA}{\boldsymbol{A}}
\newcommand{\bB}{\boldsymbol{B}}
\newcommand{\bC}{\boldsymbol{C}}
\newcommand{\bD}{\boldsymbol{D}}

\newcommand{\bN}{\boldsymbol{N}}

\newcommand{\bP}{\boldsymbol{P}}

\newcommand{\bi}{\boldsymbol{i}}
\newcommand{\bj}{\boldsymbol{j}}

\newcommand{\br}{\boldsymbol{r}}
\newcommand{\bs}{\boldsymbol{s}}
\newcommand{\bt}{\boldsymbol{t}}

\newcommand{\calC}{\mathcal{C}}
\newcommand{\calD}{\mathcal{D}}

\newcommand{\calF}{\mathcal{F}}
\newcommand{\calG}{\mathcal{G}}
\newcommand{\calH}{\mathcal{H}}

\newcommand{\calL}{\mathcal{L}}
\newcommand{\calM}{\mathcal{M}}

\newcommand{\calP}{\mathcal{P}}

\newcommand{\calR}{\mathcal{R}}
\newcommand{\calS}{\mathcal{S}}
\newcommand{\calT}{\mathcal{T}}

\DeclareMathOperator{\CNN}{CNN}
\DeclareMathOperator{\diag}{diag}
\DeclareMathOperator{\blkdiag}{blkdiag}

\newcommand\signals[2]{\ell_{2e}^{#2}(\bbN_0^{#1})}

\definecolor{mycolor1}{rgb}{0.00000,0.44700,0.74100}%
\definecolor{mycolor2}{rgb}{0.85000,0.32500,0.09800}%
\definecolor{mycolor3}{rgb}{0.92900,0.69400,0.12500}%
\definecolor{mycolor4}{rgb}{0.49400,0.18400,0.55600}%
\definecolor{mycolor5}{rgb}{0.6980,0.7725,0.8980}%
\definecolor{mycolor6}{rgb}{0.7686,0.8745,0.7020}%
\definecolor{mycolor7}{rgb}{1.0,0.9804,0.8039}%

\newcommand{\lsb}{[}
\newcommand{\rsb}{]}

\begin{document}

\title{Lipschitz constant estimation for general neural network architectures using control tools}

\author{Patricia Pauli$^1$, Dennis Gramlich$^2$, and Frank Allg\"ower$^1$
\thanks{*This work was funded by Deutsche Forschungsgemeinschaft (DFG, German Research Foundation) under Germany's Excellence Strategy - EXC 2075 - 390740016 and under grant 468094890. The authors thank the International Max Planck Research School for Intelligent Systems (IMPRS-IS) for supporting Patricia Pauli.}% <-this % stops a space
\thanks{$^{1}$Patricia Pauli and Frank Allgöwer are with the Institute for Systems Theory and Automatic Control, University of Stuttgart, 70550 Stuttgart, Germany (email: \{patricia.pauli,frank.allgower\}@ist.uni-stuttgart.de)}%
\thanks{$^{2}$ Dennis Gramlich is with the Chair of Intelligent Control Systems, RWTH Aachen, 52074 Aachen, Germany (e-mail: dennis.gramlich@ic.rwth-aachen.de).}%
}

\maketitle

\begin{abstract}
    This paper is devoted to the estimation of the Lipschitz constant of general neural network architectures using semidefinite programming. For this purpose, we interpret neural networks as time-varying dynamical systems, where the $k$-th layer corresponds to the dynamics at time $k$. A key novelty with respect to prior work is that we use this interpretation to exploit the series interconnection structure of feedforward neural networks with a dynamic programming recursion. Nonlinearities, such as activation functions and nonlinear pooling layers, are handled with integral quadratic constraints. If the neural network contains signal processing layers (convolutional or state space model layers), we realize them as 1-D/2-D/N-D systems and exploit this structure as well. We distinguish ourselves from related work on Lipschitz constant estimation by more extensive structure exploitation (scalability) and a generalization to a large class of common neural network architectures. To show the versatility and computational advantages of our method, we apply it to different neural network architectures trained on MNIST and CIFAR-10.
\end{abstract}

\begin{IEEEkeywords}
    Neural networks, Lipschitz constant, semidefinite program.
\end{IEEEkeywords}

\section{Introduction}
\IEEEPARstart{N}{eural} networks (NNs) are successfully applied in many fields, e.g., in data analysis, pattern recognition, image and video processing, natural language processing, and learning- and perception-based control \cite{bishop1994neural,li2021survey}. Especially in safety-critical applications like autonomous driving and medical control systems, it is imperative that NNs are safe and reliable \cite{muhammad2020deep}. However, many NN architectures are prone to adversarial examples, i.e., there exist imperceptible input perturbations that drastically change the output of the NN \cite{szegedy2013intriguing}. Inspired by this problem, the field of robust NNs has emerged, offering a range of techniques for robustness certification. These include bound-propagation methods \cite{zhang2018efficient,xu2020automatic,shi2022efficiently} and analysis based on Lipschitz bounds \cite{hein2017formal, tsuzuku2018lipschitz}. Through interval bound propagation \cite{gowal2018effectiveness,zhang2018efficient,xu2020automatic} certify lower bounds on the perturbations necessary to change the classification of NNs using linear and quadratic relaxations to bound the nonlinear activation function. The Lipschitz constant serves as a sensitivity measure to input perturbations and can further be used to derive robustness certificates. Consequently, many works are concerned with Lipschitz constant estimation \cite{scaman2018lipschitz,tsuzuku2018lipschitz,combettes2020lipschitz,fazlyab2019efficient,latorre2020lipschitz,chen2020semialgebraic,shi2022efficiently,pauli2023lipschitza}.

Due to the NP-hard nature of the calculation of the true Lipschitz constant \cite{virmaux2018lipschitz,jordan2020exactly}, there is a high interest to instead find accurate upper bounds on this Lipschitz constant. Trivial methods like the product of the spectral norms of the weights \cite{szegedy2013intriguing} can cheaply be computed by the power iteration method, but the resulting bounds can be quite loose, especially for deep NNs. In contrast, loop transformation based approaches \cite{fazlyab2024certified} and semidefinite programming (SDP) based approaches focusing on $\ell_2$ \cite{fazlyab2019efficient,fazlyab2020safety} and $\ell_\infty$ \cite{latorre2020lipschitz,chen2020semialgebraic} Lipschitz bounds, respectively, provide tighter bounds at the price of a computational overhead. The key idea in \cite{fazlyab2019efficient,fazlyab2020safety} is the over-approximation of nonlinear activation functions using quadratic constraints to facilitate an SDP-based analysis. In this work, we adapt this idea, and extend and generalize previous works \cite{fazlyab2019efficient,pauli2023lipschitza,gramlich2023convolutional} to develop a general SDP-based method for Lipschitz constant estimation for a large class of NN architectures.

SDP-based methods provide the tightest bounds on the $\ell_2$ Lipschitz constant for NNs in polynomial time \cite{fazlyab2019efficient}. However, their scalability to deep state-of-the-art NNs is an open research problem which is actively investigated. \cite{dathathri2020enabling,roig2022globally} develop more scalable SDP solvers and \cite{wang2024on} do so specific to the problem of SDP-based Lipschitz constant estimation, \cite{xue2022chordal} exploit the chordal sparsity pattern of the underlying linear matrix inequality (LMI) constraint for fully connected NNs and \cite{pauli2023lipschitza,gramlich2023convolutional} exploit the structure of convolutions. In this work, we exploit (i) the structure of the individual layer types and (ii) the concatenation structure of the feedforward networks. We do the latter by taking on a dynamic programming perspective and interpreting the layers of the feedforward NN as the time-varying dynamics of a system \cite{liu2019deep,fazlyab2020safety}. This view leads to a recursive formulation of layer-wise constraints, which is computationally favorable. To address the former, we especially exploit the structure and shift invariance of convolutions as we incorporate 1-D/2-D/N-D convolutions into the SDP-based analysis using Roesser-type state space representations, i.e., state space models to describe 2-D and N-D systems \cite{roesser1975discrete,gramlich2023convolutional,pauli2024state}. In contrast to previous works \cite{fazlyab2019efficient,gramlich2023convolutional}, our approach incorporates many popular layer types including convolutional, deconvolutional and state space model layers \cite{gu2021efficiently}, residual layers \cite{he2016deep}, fully connected layers, average and maximum pooling layers, and slope-restricted and GroupSort activation function layers \cite{anil2019sorting}.  

In summary, the main contribution of this work is an SDP-based method for Lipschitz constant estimation for a general class of NNs that outperforms previous SDP-based methods in terms of scalability and yields significantly lower Lipschitz bounds than commonly used methods using spectral norm bounds. To reach our goal, we exploit control theoretic concepts such as N-D systems theory and introduce a dynamic programming perspective for the underlying problem. The remainder of the paper is organized as follows. Section~\ref{sec:problem} formally states the problem, introduces all layer definitions and state space representations for convolutions. Next, Section~\ref{sec:Lip_estimation} involves our dynamic programming based approach for Lipschitz constant estimation for NNs and Section~\ref{sec:conservatism} discusses sources of conservatism. Finally, Section~\ref{sec:experiments}  applies our method on fully connected and fully convolutional networks of different sizes and multiple NN architectures to demonstrate the versatility and improved accuracy and scalability of our approach over previous approaches. We provide easy-to-use code for all considered NN architectures and layer types.

\textbf{Notation:} By $\|\cdot\|_2$ we either mean the Euclidean norm of a vector or the $\ell_2$ norm of a signal. By $\langle\cdot,\cdot \rangle_2$ we denote the $\ell_2$ inner product. By $\bbR^n$ ($\bbR_+^n$), we mean the space of $n$-dimensional vectors with real (positive) entries. By $\bbS^n$ ($\bbS_{++}^n$), we denote (positive definite) symmetric matrices and by $\bbD^n$ ($\bbD_{++}^n$) we mean (positive definite) diagonal matrices of dimension $n$, respectively. Within our paper, we study CNNs processing image signals. For this purpose, an image is understood as a sequence $(u\lsb i_1,\ldots,i_d\rsb )$ with free variables $i_1,\ldots,i_d \in \bbN_0$. In this sequence, $u\lsb i_1,\ldots,i_d\rsb$ is an element of $\bbR^c$, where $c$ is called the \emph{channel dimension} (e.g., $c = 3$ for RGB images). The \emph{signal dimension} $d$ will usually be $d=2$ for images or $d = 3$ for medical images. The space of such signals/sequences is denoted by $\signals{d}{c} := \{ u: \bbN_0^d \to \bbR^c\}$. Images should be understood as sequences in $\signals{d}{c}$ with a finite square as support. For convenience, we will sometimes use multi-index notation for signals, i.\,e., we denote $u\lsb i_1,\ldots,i_d\rsb$ as $u\lsb\bi\rsb$ for $\bi \in \bbN_0^d$. For these multi-indices, we use the notation $\bi + \bj$ for $(i_1+j_1,\ldots,i_d+j_d)$, $\bi\bj = (i_1j_1,\ldots,i_dj_d)$ and $\bi \leq \bj$ for $i_1\leq j_1,\ldots,i_d\leq j_d$. We further denote by $[\bi,\bj] = \{ \bt \in \bbN_0^d \mid \bi \leq \bt \leq \bj \}$ the \emph{interval} of all multi-indices between $\bi,\bj \in \bbN_0^d$ and by $|[\bi,\bj]|$ the number of elements in this interval. Finally, we define the interval $[\bi,\bj[ := [\bi,\bj-1]$.

\section{Problem statement and deep neural networks} \label{sec:problem}
In this work, we understand deep NNs as a concatenation of simple functions, i.\,e., as a composition
\begin{align}
    \mathrm{NN}_\theta = \calL_l \circ \calL_{l-1} \circ \cdots \circ \calL_2 \circ \calL_1 \label{eq:NN}
\end{align}
of layers $\calL_k, k=1,\ldots,l$ where $k$ is the layer index and $\calL \in \{\calF,\calC,\calS,\sigma,\calP,\calR\}$ is either a fully connected layer $\calF$, a convolutional layer $\calC$, a state space model layer $\calS$, an activation function layer $\sigma$, a pooling layer $\calP$, or a reshaping/flattening layer $\calR$. The parameter $\theta$ of $\mathrm{NN}_\theta$ refers to the collection of parameters (weights and biases) $\theta_k$ of all the individual layers. We can also write the NN recursively as the map from $u_1$ to $y_l$ defined by
\begin{align}
    y_k &= \calL_k(u_k) & u_{k+1} &= y_k & k = 1,\ldots , l, \label{eq:inputOutputNN}
\end{align}
where $u_k \in \calD_{k-1}$ and $y_k \in \calD_k$ denote the input and the output of each layer and the real vector spaces $\calD_{k-1}$ and $\calD_k$ are the input and output domains of the layer $\calL_k$. We assume here that the layers are always chosen in such a way that the image space of $\calL_{k}$ and the domain space of $\calL_{k+1}$ coincide. Consequently, our Lipschitz constant analysis applies to any finite concatenation of layers $\calL \in \{\calF,\calC,\calS,\sigma,\calP,\calR\}$. In deep learning, the definition of a layer may sometimes refer to a composition of multiple elements of $\{\calF,\calC,\calS,\sigma,\calP,\calR\}$.  For example, a linear map is grouped with a diagonally repeated activation function or a convolutional layer, an activation function, and a pooling layer are grouped together as a layer. Our approach can handle such concatenated layer definitions, meaning that we additionally allow $\calL \in \{\sigma \circ \calF, \sigma \circ \calC, \calP \circ \sigma \circ \calC\}$ or a concatenation of even more layers, compare (cmp.) Section~\ref{sec:subnetworks}.

Regardless of the layer definition, our examples usually study convolutional neural networks $\CNN_\theta$ with the structure
\begin{align*}
    \calF_{l} \circ \sigma \circ \cdots \circ \sigma \circ \calF_{p+1} &\circ \calR \circ \cdots\\ \cdots &\circ \calP\circ \sigma  \circ \calC_{p} \circ \cdots \circ \calP\circ \sigma \circ \calC_1,
\end{align*}
typically found in image classification. These convolutional neural networks (CNNs) are composed of fully connected layers $\calF$, activation function layers $\sigma$, a flattening operation $\calR$, convolutional layers $\calC$, and (optional) pooling layers $\calP$ in the order shown above.

The goal of this work is to provide an accurate and scalable method that determines an upper bound on the Lipschitz constant of a general feedforward NN \eqref{eq:NN}, \eqref{eq:inputOutputNN}.

\begin{problem}
    \label{prob:Lip}
    For a given neural network $\mathrm{NN}_\theta$ with parameters~$\theta$, find an upper bound on the Lipschitz constant, i.\,e., find a value $\gamma\geq 0$ such that
    \begin{align*}
        \| \mathrm{NN}_\theta (u^1) - \mathrm{NN}_\theta (u^2) \|_2 \leq \gamma \| u^1 - u^2 \|_2 \quad \forall u^1,u^2 \in \calD_0.
    \end{align*}
\end{problem}

We notice that the definition of an NN \eqref{eq:NN}, \eqref{eq:inputOutputNN} resembles a dynamical system $u_{k+1} = \calL_k(u_k)$ with state $u_k$. The interpretation of an NN as a dynamical system with time-varying dynamics \cite{liu2019deep,fazlyab2020safety} is very powerful because it enables us to use tools from control and systems theory to analyze properties of NNs. However, we stress that this interpretation should be taken with caution, since the inputs $u_k$ and $u_j$ for $k \neq j$ usually live in different spaces $\calD_{k-1}$ and $\calD_{j-1}$. As we will see in Section~\ref{sec:layer-definitions}, we allow signal spaces $\calD_k = \signals{d_k}{c_k}$ of $d_k$-dimensional signals as well as vector spaces $\calD_k = \bbR^{c_k}$. Also the vector (= channel) dimension $c_k$ may differ from one layer to another. In addition, the nature of the mappings $\calL_k$ and $\calL_j$ for $k \neq j$ can be completely different, including linear and nonlinear mappings. In some less heterogeneous examples, e.g., fully connected networks $\calF_l\circ\sigma\circ\dots\circ\sigma\circ\calF_1$, or fully convolutional networks and subnetworks $\calC_l\circ\sigma\circ\dots\circ\sigma\circ\calC_1$, the interpretation as dynamical systems is, however,  more natural and it is common practice to design a deep backbone of NNs of the same layer type \cite{hu2023scaling}.

It is the defining selling point of our work that we \emph{exploit the structure of each layer}, as well as the \emph{composition structure of the NN itself} using different perspectives and methods from control. The NN is described as a concatenation of its layers, which is now followed by a definition of each individual layer in Section~\ref{sec:layer-definitions}. Subsequently, in Section~\ref{sec:statespace}, we introduce state space representations for convolutional layers.

\subsection{Layer definitions}\label{sec:layer-definitions}
\paragraph*{Convolutional layer}\label{sec:layer_def_conv}
A convolutional layer $\calC_k$ with layer index $k$ is a mapping from $\calD_{k-1} = \signals{d_{k-1}}{c_{k-1}}$ to $\calD_k = \signals{d_k}{c_k}$ which is defined by a convolution kernel $K_k\lsb\bt\rsb\in\bbR^{c_{k} \times c_{k-1}}$ for $0\leq\bt\leq\br_k$ and a bias $b_k \in \bbR^{c_{k}}$. We write
\begin{align}
    y_{k}[\bi] = b_k + \sum_{0 \leq \bt \leq \br_k} K_k[\bt] u_k[\bi - \bt], \label{eq:conv_alt}
\end{align}
where $u_k[\bi - \bt]$ is set to zero if $\bi - \bt$ is not in the domain of $u_k\lsb\cdot\rsb$, which accounts for possible zero-padding. A convolutional layer retains the dimension $d_{k-1}=d_k=d$ but it may change in channel size from $c_{k-1}$ to $c_k$. The multi-index $\br_k \in \bbN_0^d$ defines the size of the kernel $K_k\lsb\cdot\rsb$.

This compact description of a convolution \eqref{eq:conv_alt} includes N-D convolutions ($d=N$). For instance, a 1-D convolution ($d=1$) 
\begin{align}
    y_{k}[i] = b_k + \sum_{t=0}^{r_k} K_k[t] u_k[i - t] \label{eq:conv_1D}
\end{align}
operates on a 1-D signal, e.g., a time signal, a 2-D convolution ($d=2$)
\begin{align}
    y_{k}[i_1,i_2] = b_k + \sum_{t_1 =0}^{r_{k1}}\sum_{t_2 =0}^{r_{k2}} K_k[t_1,t_2] u_k[i_1 - t_1,i_2-t_2] \label{eq:conv_2D}
\end{align}
operates on signals with two propagation dimensions, e.g., images, and an N-D convolution considers inputs with even more input dimensions, e.g., 3-D convolutions may be used for videos or 3-D medical images. 

An extension of the convolutional layer \eqref{eq:conv_alt} is a strided convolutional layer $\calC_{\bs_k}$ with stride $\bs_k \in \bbN^d$. For convolutions with stride $\bs_k = (s_{k1},\ldots, s_{kd})$, the output is not given by \eqref{eq:conv_alt}, but by
\begin{align}
    y_{k}[\bi] = b_k + \sum_{0 \leq \bt \leq \br_k} K_k[\bt] u_k[\bs_k\bi - \bt]. \label{eq:conv_strided}
\end{align}
This means that we always shift the kernel by $s_{k1},\ldots, s_{kd}$ along the respective signal dimension $1,\ldots, d$.

\paragraph*{Fully connected layer} In the case of a fully connected layer $\calF_k$ with layer index $k$ the domain and image spaces are $\calD_{k-1} = \bbR^{c_{k-1}}$ and $\calD_k = \bbR^{c_k}$, i.\,e., there are only the channel dimensions $c_{k-1}, c_k$ (= number of neurons of the input and output layer) and no signal dimensions $d_{k-1}, d_k$, i.\,e., $d_{k-1}=d_k=0$. We define a fully connected layer as an affine function
\begin{align}
    \calF_k : \bbR^{c_{k-1}} \to \bbR^{c_{k}}, ~ u_k \mapsto y_k = b_k + W_k u_k. \label{eq:linearLayer}
\end{align}
The vector $b_k \in \bbR^{c_{k}}$ is called the bias and $W_k \in \bbR^{c_{k} \times c_{k-1}}$ is called the weight matrix.

\begin{remark}\label{rem:correspondence}
Note that the fully connected layer is a special case of a convolutional layer for $d_{k-1} = d_k = 0$. Indeed, $\bbR^{c_{k-1}} \cong \ell_{2e}^{c_{k-1}}(\{0\}) = \signals{0}{c_{k-1}}$ and $\bbR^{c_{k}} \cong \signals{0}{c_k}$. Furthermore, we can understand $W_k$ as the convolution kernel which, in the case $d_k = 0$, is given by $K_k[0] := W_k$. Consequently, all results presented in this work for convolutional layers automatically also hold for fully connected layers. 
\end{remark}

\paragraph*{Activation function layer}\label{sec:layer_def_act}

An activation function layer $\sigma$ can be applied to any of our domain spaces $\calD_{k-1} = \bbR^{c_{k-1}}$ or $\calD_{k-1} = \signals{d_{k-1}}{c_{k-1}}$, but it requires $\calD_{k} \cong \calD_{k-1}$. We consider activation functions which are defined by scalar activation functions $\sigma : \bbR \to \bbR$, applied element-wise if applied to a vector $u_k \in \bbR^{c_k}$. To this end, for finite-dimensional vector spaces, $\sigma$ is identified with the function
\begin{align*}
    \sigma: \bbR^{c_k} \to \bbR^{c_k},~ u_{k} \mapsto y_{k} = %\sigma(u_{k}) = 
    \begin{bmatrix}
        \sigma(u_{k1}) & \cdots & \sigma(u_{kc_k})
    \end{bmatrix}^\top.
\end{align*}
We further lift the scalar activation function to signal spaces by defining the activation function layer on $\signals{d_k}{c_k}$ as the function $\sigma : \signals{d_k}{c_k} \to \signals{d_k}{c_k}$,
\begin{align*}
    (u_k[\bi])_{\bi \in \bbN_0^{d_k}} \mapsto (y_k[\bi])_{\bi \in \bbN_0^{d_k}} = (\sigma(u_k[\bi]))_{\bi \in \bbN_0^{d_k}}.
\end{align*}

\paragraph*{Pooling layer} Pooling layers are downsampling operations from $\calD_{k-1} = \signals{d_{k-1}}{c_{k-1}}$ to $\calD_k = \signals{d_k}{c_k}$ with $d_{k-1} = d_k=d$ and $c_{k-1} = c_k$ that take a batch of input signal entries $(u_k\lsb\bs_k\bi + \bt\rsb \mid \bt \in [0,\br_k])$ and map them channel-wise into one single output signal entry $y[\bi]$. The two most common pooling layers are average pooling $\calP^{\mathrm{av}} : \signals{d}{c_k} \to \signals{d}{c_k}$,
\begin{align*}
    y_k\lsb\bi\rsb :=& \mathrm{mean} (u_k\lsb\bs_k\bi - \bt\rsb \mid \bt \in [0,\br_k])\\
    =& \frac{1}{|[0,\br_k]|} \sum_{0 \leq \bt \leq \br_k} u_k\lsb\bs_k\bi - \bt\rsb
\end{align*}
and maximum pooling $\calP^{\mathrm{max}} : \signals{d_k}{c_k} \to \signals{d_k}{c_k}$,
\begin{align*}
    y_k\lsb\bi\rsb &:= \max (u_k\lsb\bs_k\bi - \bt\rsb \mid \bt \in [0,\br_k]),
\end{align*}
where the maximum is applied channel-wise. For most pooling layers the kernel size and the stride coincide ($\br_k+1=\bs_k$), yet sometimes, e.g., in AlexNet \cite{krizhevsky2012imagenet}, $\br_k+1>\bs_k$ is chosen.

\paragraph*{Flattening operator} Flattening is a pure reshaping operation, which merges the signal dimensions into the channel dimension. Note that the mapping is not injective, i.e., a square batch $(u_k\lsb\bi\rsb \mid 0 \leq \bi < \bN_{k})$, for example a finite-dimensional image,  is reshaped into the channel dimension and the remaining entries (mostly zeros) are discarded. The typical flattening operation is a vectorization given by
\begin{align*}
	\calR & : \signals{d_{k-1}}{c_{k-1}} \to \bbR^{|\lsb0,\bN_{k}\lsb| \cdot c_{k-1}}, (u_k[\bi] )_{\bi\in\bbN_0^{d_k}} \mapsto y_k,
\end{align*}
where $y_k$ is a stacked vector of $u_k\lsb\bi\rsb,~0\leq \bi < \bN_{k}$, i.\,e., $y_k^\top=\begin{bmatrix} u_k\lsb 0,\dots,0\rsb^\top & \dots & u_k\lsb N_{k1},\dots,N_{kd}\rsb^\top \end{bmatrix}$. We could also define flattening operators $\signals{d_{k-1}}{c_{k-1}} \to \signals{d_k}{c_k}$ with $1 \leq d_{k} < d_{k-1}$ contracting only some of the signal dimensions and not all at once. For example, we can flatten 2-D signals into 1-D signals ($d_k=1$) or into vectors with $d_k=0$.

\paragraph*{State space model layer}
State space model layers have recently gained popularity in the machine learning community \cite{gu2021efficiently}. We define a state space model layer with layer index $k$ as an affine time-invariant system $\calS_k:\ell_{2e}^{c_{k-1}}(\bbN_0^{1})\to\ell_{2e}^{c_{k}}(\bbN_0^{1})$
\begin{equation*}
    \begin{bmatrix}
        x_k[i+1] \\
        y_k[i]
    \end{bmatrix}
    =
    \begin{bmatrix}
        0 & A_k & B_k\\
        g_k & C_k & D_k
    \end{bmatrix}
    \begin{bmatrix}
        1 \\
        x_k[i] \\
        u_k[i]
    \end{bmatrix},    
\end{equation*}
where $x_k[i]\in\bbR^{n_k}$ denotes the state. The state space model is characterized by some matrices ($A_k,B_k,C_k,D_k,g_k$) of appropriate dimensions.

\subsection{State space representations for convolutions}\label{sec:statespace}
In the machine learning literature, convolutional layers are usually represented as in \eqref{eq:conv_alt} using a convolution kernel \cite{Goodfellow-et-al-2016}. However, state space realizations have proven to be more amenable to analysis using tools from robust control than such kernel (impulse response) representations \cite{gramlich2023convolutional}. In the control engineering literature, mappings from $\signals{d}{c_{k-1}}$ to $\signals{d}{c_k}$ are known as N-D systems and, as it is shown in \cite{bett1997linear}, N-D systems with rational transfer functions admit a N-D state space representation. N-D convolutions are finite impulse response (FIR) filters, for which we in the following introduce state space realizations of the Roesser type \cite{roesser1975discrete}. Throughout this section, we drop the layer index $k$ to improve readability. If we refer to the previous layer we use the subscript ``$-$'', e.g., $c_-$ means $c_{k-1}$.

\begin{definition}[Roesser model]
	\label{def:RoesserSystem}
	An affine N-D system $\signals{d}{c_{-}} \to \signals{d}{c}, (u\lsb \bi \rsb ) \mapsto (y\lsb \bi \rsb )$ is described by a Roesser model as
	\begin{align}
		\begin{bmatrix}
			x_1\lsb \bi + e_1 \rsb \\
                \vdots\\
			x_{d}\lsb \bi + e_d \rsb  \\
			y\lsb \bi \rsb
		\end{bmatrix}
		=
		\begin{bmatrix}
			0 & A_{11} & \cdots & A_{1d} & B_1\\
                \vdots & \vdots & \ddots & \vdots & \vdots\\
			0 & A_{d1} & \cdots & A_{dd} & B_d\\
			g & C_1 & \cdots & C_d & D
		\end{bmatrix}
		\begin{bmatrix}
			1\\
			x_1\lsb \bi \rsb \\
                \vdots\\
			x_d\lsb \bi \rsb\\
			u\lsb \bi \rsb
		\end{bmatrix},
		\label{eq:RoesserSys}
	\end{align}
    where $e_i$ denotes the unit vector with $1$ in the $i$-th position. Here, the collection of matrices $A_{11}$, $\ldots$ $C_d$, $D$ is called state space representation of the system, $x_1\lsb \bi \rsb \in \bbR^{n_1}$, $\ldots$ $x_d\lsb \bi \rsb \in \bbR^{n_d}$ are the states, $u\lsb \bi \rsb \in \bbR^{c_{-}}$ is the input and $y\lsb \bi \rsb \in \bbR^{c}$ is the output of the system. We call \eqref{eq:RoesserSys} a linear time-invariant N-D system ($N = d$) if $g = 0$. Otherwise, we call the system affine time-invariant.
\end{definition}

We define
    \begin{equation*}
        \left[
        \begin{array}{c|c|c}
		0 & \bA & \bB\\ \hline
		g & \bC & \bD
        \end{array}\right]:=
        \left[
        \begin{array}{c|ccc|c}
            0 & A_{11} & \cdots & A_{1d} & B_1\\
            \vdots & \vdots & \ddots & \vdots & \vdots\\
		0 & A_{d1} & \cdots & A_{dd} & B_d\\\hline
		g & C_1 & \cdots & C_d & D
	\end{array}\right].
    \end{equation*}

Realizing 1-D convolutions in state space is straightforward \cite{pauli2023lipschitza}. For the important layer type of 2-D convolutions ($d=2$), i.e., the 2-D system
	\begin{align}
		\begin{bmatrix}
			x_1\lsb i_1+1, i_2 \rsb \\
			x_{2}\lsb i_1, i_2+1 \rsb  \\
			y\lsb i_1,i_2 \rsb
		\end{bmatrix}
		=
		\begin{bmatrix}
			0 & A_{11} & A_{12} & B_1\\
			0 & A_{21} & A_{22} & B_2\\
			g & C_1 & C_2 & D
		\end{bmatrix}
		\begin{bmatrix}
			1\\
			x_1\lsb i_1, i_2 \rsb \\
			x_2\lsb i_1, i_2 \rsb \\
			u\lsb i_1, i_2 \rsb
		\end{bmatrix},
		\label{eq:RoesserSys2D}
	\end{align}
we use the construction presented in \cite{pauli2024state} as stated in Lemma~\ref{lem:min_real_2D}.

\begin{lemma}[Realization of 2-D convolutions \cite{pauli2024state}]\label{lem:min_real_2D}
    Consider a convolutional layer $\calC : \signals{2}{c_{-}} \to \signals{2}{c}$ with representation \eqref{eq:conv_2D} characterized by the convolution kernel $K$ and the bias $b$. This layer is realized in state space by the matrices
    \begin{align*}
        &\left[
        \begin{array}{c|c}
            A_{12} & B_1 \\ \hline
            C_2 & D
        \end{array}
        \right]
        =
        \left[
        \begin{array}{ccc|c}
            K\lsb r_1,r_2 \rsb & \cdots & K \lsb r_1,1 \rsb & K \lsb r_1,0 \rsb \\
            \vdots & \ddots & \vdots & \vdots \\
            K\lsb 1,r_2 \rsb & \cdots & K \lsb 1,1 \rsb & K \lsb 1,0 \rsb \\ \hline
            K\lsb 0,r_2 \rsb & \cdots & K \lsb 0,1 \rsb & K \lsb 0,0 \rsb \\
        \end{array}
        \right],\\
        &\left[
        \begin{array}{c}
            A_{11} \\ \hline
            C_1
        \end{array}
        \right]
        =
        \left[
        \begin{array}{cc}
            0 & 0\\
            I_{c(r_1-1)} & 0\\ \hline
            0 & I_{c}
        \end{array}
        \right],~A_{21} = 0, \quad g = b,\\
        &\left[
        \begin{array}{c|c}
            A_{22} & B_2
        \end{array}
        \right]
        =
        \left[
        \begin{array}{cc|c}
            0 & I_{c_-(r_2-1)} & 0\\
            0 & 0 & I_{c_-}
        \end{array}
        \right],
    \end{align*}
    where $K[i_1,i_2]\in\bbR^{c\times c_{-}},~i_1\in\lsb 0,r_1\rsb,~i_2\in\lsb 0,r_2\rsb$. The state signals $(x_1\lsb i_1,i_2\rsb)_{i_1,i_2 \in \bbN_0}$ with $x_1\lsb i_1,i_2\rsb \in \bbR^{n_1}$, $n_1 = c r_1$, and $(x_2\lsb i_1,i_2\rsb)_{i_1,i_2 \in \bbN_0}$ with $x_2\lsb i_1,i_2\rsb \in \bbR^{n_2}$, $n_2 = c_{-} r_2$ are given inductively by \eqref{eq:conv_2D} with $x_1\lsb 0,i_2\rsb = 0$ for all $i_2\in\bbN_0$, and $x_2\lsb i_1,0\rsb = 0$ for all $i_1\in\bbN_0$.
\end{lemma}

\begin{proof}
    See \cite[Theorem 1]{pauli2024state} for a proof.
\end{proof}

\begin{remark}
    Finding a mapping from $K$ to $(\bA,\bB,\bC,\bD)$ for N-D convolutions and dilated convolutions is also possible, see~\cite{pauli2024state}. We discuss state space representations for strided convolutions in Appendix~\ref{sec:strided_convolutions}.
\end{remark}

\begin{remark}
    Representing a convolution in state space requires the choice of a propagation direction for both dimensions. Usually, for image inputs we pick the upper left corner as the origin with $i_1=i_2=0$. However, any other corner and corresponding propagation directions can also be chosen to represent the convolution equivalently. For state space model layers the propagation dimension, i.e., time is predefined, and cannot be changed.
\end{remark}

\section{Lipschitz constant estimation}\label{sec:Lip_estimation}
To address Problem \ref{prob:Lip} of estimating the Lipschitz constant of an NN, the interpretation \eqref{eq:inputOutputNN} of NNs as dynamical systems $u_{k+1} = \calL_k(u_k)$ is utilized. Namely, we pose the problem of estimating the Lipschitz constant of $\mathrm{NN}_\theta$ as the dynamic optimization problem
\begin{align}
    \min_{\gamma \in \bbR} ~&~ \gamma \label{eq:LipschitzDP}\\
    \mathrm{s.t.}~&~ \|y_l^1 -y_l^2\|_2 \leq \gamma\|u_1^1 - u_1^2\|_2, & \forall u_1^1,u_1^2 \in \calD_0, \nonumber\\
    ~&~ y_k^1 = \calL_k (u_k^1), ~ y_k^2 = \calL_k (u_k^2), & k = 1,\ldots,l, \nonumber\\
    ~&~ u_{k+1}^1 = y_k^1, ~ u_{k+1}^2 = y_k^2, & k = 1,\ldots,l-1. \nonumber
\end{align}

The advantage of the recursive formulation \eqref{eq:LipschitzDP} is that it can be solved using a dynamic programming approach. Namely, interpreting the equality constraints
\begin{align*}
    (u_{k+1}^1,u_{k+1}^2) = (y_k^1,y_k^2) = (\calL_k(u_k^1),\calL_k(u_k^2))
\end{align*}
as the dynamics of a system with tuples $(u_k^1,u_k^2)$ as states, we can recursively define value functions
\begin{align}
    \begin{split}
    V_l(y_l^1,y_l^2) &= \| y_l^1 - y_l^2\|_2^2, \quad y_l^1,y_l^2 \in \calD_l\\
    V_{k-1}(u_k^1,u_k^2) &= V_k(\calL_k(u_k^1),\calL_k(u_k^2)),~ u_k^1,u_k^2 \in \calD_{k-1},
    \end{split} \label{eq:DP_recursion}
\end{align}
for $k = 1,\ldots,l$, starting from the $l$-th layer.
Drawing from dynamic programming, we can think of $\| y_l^1 - y_l^2\|_2^2$ as a terminal cost and the stage cost as being zero. We obtain that \eqref{eq:LipschitzDP} is equivalent to finding the smallest $\gamma \in \bbR_+$ such that $V_0(u_1^1,u_1^2) \leq \gamma^2 \|u_1^1 - u_1^2\|_2^2$ for all $u_1^1,u_1^2 \in \calD_0$. We can therefore pose \eqref{eq:LipschitzDP} as the optimization problem
\begin{subequations}
\label{eq:metricOpt}
\begin{align}
    \min_{\gamma, V_1,\ldots,V_l} ~&~ \gamma \label{eq:metricOpt_Cost}\\
    \mathrm{s.t.} ~&~ V_l(y^1_l,y^2_l) \geq \|y_l^1 - y_l^2\|_2^2  \label{eq:metricOpt_terminal} \\
    ~&~ V_{k-1}(u_k^1,u_k^2) \geq V_k(\calL_k(u_k^1),\calL_k(u_k^2)), ~ k=l,\ldots,2  \label{eq:metricOpt_DP}\\
    ~&~\gamma^2 \|u_1^1 - u_1^2\|_2^2 \geq V_1(\calL_1(u_1^1),\calL_1(u_1^2)) \label{eq:metricOpt_initial} %\\
\end{align}
\end{subequations}
over value functions, where \eqref{eq:metricOpt_terminal} to \eqref{eq:metricOpt_initial} must hold for all $u_k^1,u_k^2 \in \calD_{k-1}$, $k = 1,\ldots, l$. %$u_1^1,u_1^2\in\calD_0$ and $u^1_k,u^2_k$, $k=1,\dots,l$ are obtained by \eqref{eq:inputOutputNN}.
The equivalence of \eqref{eq:LipschitzDP} and \eqref{eq:metricOpt} can be checked by chaining the inequalities \eqref{eq:metricOpt_initial}, \eqref{eq:metricOpt_DP}, \eqref{eq:metricOpt_terminal} as
\begin{align}
    &\gamma^2 \|u_1^1 - u_1^2\|_2^2 \geq V_1(\calL_1(u_1^1),\calL_1(u_1^2)) = V_1(u_2^1,u_2^2) \label{eq:inequChain}\\
    &\geq V_2(\calL_2(u_2^1),\calL_2(u_2^2)) \geq \ldots \geq V_l(y^1_l,y^2_l) \geq \|y_l^1 - y_l^2\|_2^2. \nonumber
\end{align}
Here, the sequence of inequalities \eqref{eq:inequChain} shows that the constraints of \eqref{eq:metricOpt} imply those of \eqref{eq:LipschitzDP}. In addition, inserting the value functions \eqref{eq:DP_recursion} into the constraints of \eqref{eq:LipschitzDP} yields \eqref{eq:inequChain} and shows that the constraints of \eqref{eq:LipschitzDP} also imply \eqref{eq:metricOpt_terminal} to \eqref{eq:metricOpt_initial}. Furthermore, as $\|y_l^1 - y_l^2\|_2^2\geq 0$, we can deduce from the chain of inequalities \eqref{eq:inequChain} that the value functions $V_k,~k=1,\dots,l$ are positive definite. Note that Problem \eqref{eq:metricOpt} involves $l$ constraints of the form \eqref{eq:metricOpt_DP} and \eqref{eq:metricOpt_initial} and that this layer-wise splitting achieved by introducing value functions is computationally favorable over using one large and sparse constraint for the whole NN.

Still, at the present state, \eqref{eq:metricOpt} is an intractable problem due to the optimization over the infinite-dimensional objects (functions) $V_k$ and the infinitely many constraints \eqref{eq:metricOpt_DP} which must hold for all $u_k^1,u_k^2\in \calD_{k-1}$. For this reason, we refer to a very common relaxation from the control literature, namely, quadratic value functions. To this end, we constrain the functions $V_k$ to be of the form
\begin{align}\label{eq:quadratic_value_function}
    V_k(y_k^1,y_k^2) &= V_{X_k}(y_k^1,y_k^2) := \langle y_k^1 - y_k^2, X_k (y_k^1 - y_k^2) \rangle_2
\end{align}
for linear self-adjoint, positive definite operators $X_k$ on $\calD_k$, i.e., $\langle y_k^1 - y_k^2, X_k (y_k^1 - y_k^2) \rangle_2=\langle X_k^* (y_k^1 - y_k^2), y_k^1 - y_k^2
\rangle_2$ for all $y_k^1,y_k^2\in\calD_k$ \cite{akhiezer2013theory}. In the case $\calD_k = \bbR^{c_k}$ we may simply assume that the operators $X_k$ are in matrix representation and obtain 
\begin{equation*}
V_{X_k}(y_k^1,y_k^2) = (y_k^1 - y_k^2)^\top X_k (y_k^1 - y_k^2).
\end{equation*}
In the case $\calD_k = \signals{d_k}{c_k}$, we can represent $X_k$ in terms of a sequence of matrices $(X_k^\prime\lsb \bi,\bj \rsb)_{\bi,\bj \in \bbN_0^{d_k}}$, $X_k^\prime[\bi,\bj]\in\bbR^{c_k\times c_k}$ by
\begin{align*}
    V_{X_k}(y_k^1,y_k^2) = \sum_{\bi,\bj \in \bbN_0^{d}}(y_k^1\lsb \bi \rsb - y_k^2\lsb \bi \rsb )^\top X_k^\prime\lsb \bi,\bj \rsb (y_k^1 \lsb \bj \rsb - y_k^2 \lsb \bj\rsb).
\end{align*}
For this representation, self adjointness of $X_k$ means that $X^\prime_k$ is symmetric, i.\,e., $X_k^\prime\lsb \bi,\bj\rsb = X_k^\prime\lsb \bi,\bj\rsb^\top = X_k^\prime\lsb \bj,\bi \rsb$. This relaxation is a first step towards rendering the optimization tractable. In particular, with the assumption $V_k = V_{X_k}$ we can replace the constraint \eqref{eq:metricOpt_DP} at the $k$-th layer in \eqref{eq:metricOpt} by
\begin{equation}\label{eq:metricOpt_DP_relaxed} 
    V_{X_{k-1}}(u_k^1,u_k^2) \geq V_{X_k}(\calL_k(u_k^1),\calL_k(u_k^2)).
\end{equation}
A second step involves deriving sufficient LMI conditions that imply \eqref{eq:metricOpt_DP_relaxed} for all layer types or subnetworks $\calL \in \{\calF,\calC,\sigma,\calP,\calR\} \cup \{ \sigma \circ \calF, \sigma \circ \calC, \calP \circ \sigma \circ \calC \} \cup \{ \text{residual layer}\}$. We denote these LMIs by $\mathcal{G}_{k}(X_k,X_{k-1},\nu_k) \succeq 0$, $k=1,\dots,l$ with respective slack variables $\nu_k$ which relaxes the optimization problem~\eqref{eq:metricOpt} to
\begin{align}
    \min_{X_0,\ldots,X_l,\nu_1,\ldots,\nu_l,\rho^2} ~&~ \rho^2 \label{eq:finalSDP}\\
    \mathrm{subject~to} ~&~ X_0 = \rho^2 I,\nonumber\\
    ~&~ \mathcal{G}_{k}(X_k,X_{k-1},\nu_k) \succeq 0, & k=1,\ldots,l, \nonumber\\
    ~&~ X_k \in \mathcal{H}^y_{\calL_k} \cap \mathcal{H}^u_{\calL_{k+1}}, & k = 1,\ldots,l-1,\nonumber\\
    ~&~ X_l = I . \nonumber
\end{align}
For certain layers, we impose additional restrictions on $X_k$ to enhance tractability of the problem, whose nature will be discussed in the next sections for the individual layer types. These restrictions are denoted by $\calH^y_{\calL_k}$ and $\calH^u_{\calL_k}$, corresponding to output and input restrictions, respectively. Notice that each operator $X_k$ has to satisfy the restrictions $\mathcal{H}^y_{\calL_k}$ and  $\mathcal{H}^u_{\calL_{k+1}}$ of two layers. The resulting optimization problem \eqref{eq:finalSDP} is an SDP with one LMI constraint per layer, where the index $k=1,\dots,l$ counts through all considered layers/subnetworks. 

In what follows, we provide detailed derivations of the LMI constraints $\mathcal{G}_{k}(X_k,X_{k-1},\nu_k) \succeq 0$ in \eqref{eq:finalSDP}, covering all relevant layer types and subnetworks. For ease of exposition, the layer indices $k$ are omitted and the subscript ``$-$'' is used to refer to the previous layer, e.g., $X_-$ is short for $X_{k-1}$.

\subsection{The convolutional layer}

If $\calL = \calC$ is a convolutional layer, then $\calD_{-} = \signals{d}{c_{-}}$ and $\calD = \signals{d}{c}$. % for $d = d_k = d_{-}.
Convolutional layers described by \eqref{eq:conv_alt} are shift-invariant mappings and a similar property is imposed on the operators $X$. Particularly, we require that $X_{-}^\prime$ and $X^\prime$ are of the form
\begin{align}\label{eq:restriction_conv}
    X_{-}^\prime\lsb \bi,\bj \rsb = \begin{cases}
        \widetilde{X}_{-} & \bi = \bj\\
        0 & \bi \neq \bj,
    \end{cases}
    \quad 
    X^\prime\lsb \bi,\bj \rsb = \begin{cases}
        \widetilde{X} & \bi = \bj\\
        0 & \bi \neq \bj,
    \end{cases}
\end{align}
i.\,e., these operators are parametrized by matrices $\widetilde{X} \in \bbS_{++}^{c}$ and $\widetilde{X}_{-} \in \bbS_{++}^{c_{-}}$ in a \emph{block-diagonally repeated} fashion. We denote the convolution-specific restriction in \eqref{eq:restriction_conv} by $X \in \calH_{\calC}^y$ and $X_{-} \in \calH_{\calC}^u$. 

This parameterization is a key to deriving efficient LMI representations of the inequality \eqref{eq:metricOpt_DP_relaxed}. The assumption that $X$ is a block-diagonal, shift-invariant operator is a restriction that may introduce conservatism. However, the use of e.g. block-diagonal Lyapunov functions for N-D systems \cite{xiao1997stability}, or, more generally, time-invariant value functions for time-invariant systems \cite{bertsekas2005dynamic} has proven useful in the control literature and often comes with only moderate conservatism.
%
%This restriction might seem confining at first sight, but it renders the problem computationally tractable and leverages the structure of convolutional layers, i.e., the shift invariance.
%
Thanks to the diagonally repeated parameterization of $X$, we relax \eqref{eq:metricOpt_DP_relaxed} as an LMI as follows.

\begin{lemma}\label{lem:diss_CNN}
    Consider a convolutional layer $\calL=\calC$. For some operators $X \in \calH_{\calC}^y$ and $X_{-} \in \calH_{\calC}^u$, the convolutional layer \eqref{eq:conv_alt} represented by a Roesser model \eqref{eq:RoesserSys} satisfies \eqref{eq:metricOpt_DP_relaxed} if there exist symmetric matrices $P_m \in \bbS_{++}^{n_{m}}$, $\bP = \blkdiag (P_1,\ldots , P_d)$ such that
    \begin{equation}\label{eq:cert_CNN}
        \begin{bmatrix}
            \bP & 0\\
            0 & \widetilde{X}_{-}
        \end{bmatrix}
        -
        \begin{bmatrix}
            \bA & \bB\\
            \bC & \bD
        \end{bmatrix}^\top
        \begin{bmatrix}
            \bP & 0\\
            0 & \widetilde{X}
        \end{bmatrix}
        \begin{bmatrix}
            \bA & \bB\\
            \bC & \bD
        \end{bmatrix}\succeq 0.
    \end{equation}
\end{lemma}
\begin{proof}
    A proof of Lemma~\ref{lem:diss_CNN} is given in \cite[Theorem 4]{gramlich2023convolutional} for 2-D systems and in Appendix \ref{app:proof_lem_diss_CNN} for N-D systems.
\end{proof}

The inequality \eqref{eq:cert_CNN} is an instance of $\calG(X_{-},X,\nu) \succeq 0$ in \eqref{eq:finalSDP} with slack variables $\nu = \bP$. 
We further note that \eqref{eq:cert_CNN} in Lemma \ref{lem:diss_CNN} is an exact characterization of \eqref{eq:metricOpt_DP_relaxed} for our restricted value function parameters $X \in \calH_\calC^y$, $X_{-}\in \calH_\calC^u$ in the cases $d = 0,1$. This is because \eqref{eq:metricOpt_DP_relaxed} is equivalent to dissipativity of the layer $\calL$ with respect to (w.r.t.) the supply rate $(u,y) \mapsto u^\top \widetilde{X}_{-} u - y^\top \widetilde{X} y$. For $d \leq 1$, $\calL$ is at most a 1-D system and dissipativity of 1-D systems is exactly characterized by the dissipation inequality \eqref{eq:cert_CNN} \cite{scherer2000linear}. For $d \geq 2$, this is no longer the case and the conservatism of \eqref{eq:cert_CNN} might stem from the block-diagonal structure of $\bP$ and also depends on the choice of the realization, i.e., the choice of the matrices $(\bA,\bB,\bC,\bD)$ \cite{gramlich2023convolutional}. State space representations are non-unique and interestingly, \cite{gramlich2023convolutional} empirically discovered that a non-minimal state space representation can lead to less conservative results. For the interested reader, we discuss the extension of Lemma~\ref{lem:diss_CNN} to strided convolutions \eqref{eq:conv_strided} in Appendix~\ref{sec:strided_convolutions}.

Note the special structure of \eqref{eq:cert_CNN}. If we understand $\widetilde{X}$/$\widetilde{X}_{-}$ as another block of $\bP$, then this matrix inequality is a Lyapunov inequality for a $(d+1)$-D system \cite{xiao1997stability}, an inequality that often appears in stability analysis of dynamical systems. The $(\bA,\bB,\bC,\bD)$ block plays the role of the $A$-matrix. This system is time-varying along the $k$-axis, which should be viewed as the time axis, and time-invariant along all other axes, which should be viewed as space axes.

\begin{figure}[!t]
\centering
\subfloat[]{\includegraphics[width=0.15\textwidth]{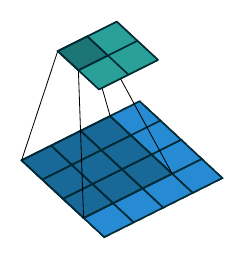}%
\label{fig:no_padding}}
\hfil
\subfloat[]{\includegraphics[width=0.15\textwidth]{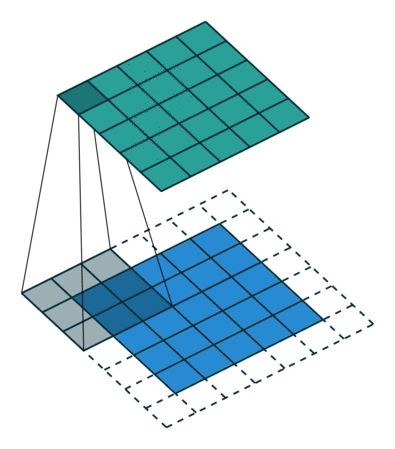}%
\label{fig:same_padding}}
\subfloat[]{\includegraphics[width=0.15\textwidth]{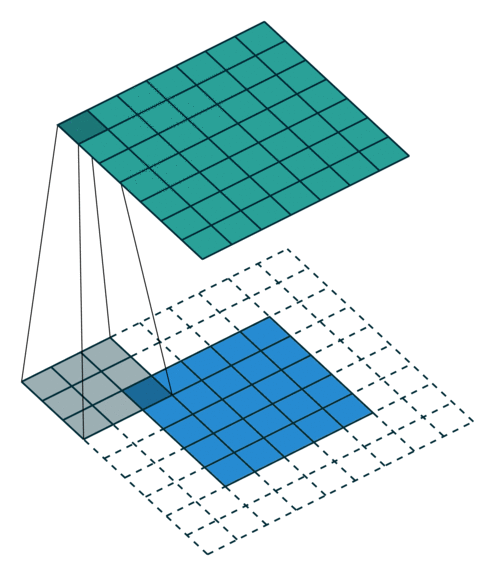}%
\label{fig:full_padding}}
\caption{(a) No padding, (b) same padding, (c) full padding for a $3\times3$ kernel \cite{dumoulin2016guide}.}
\label{fig:padding}
\end{figure}

Another design choice for convolutional layers is the kind of \emph{zero-padding} that is used. There are different kinds of padding as shown in Fig.~\ref{fig:padding} \cite{dumoulin2016guide}. We distinguish between full padding which increases the output dimension, same padding which preserves it and no/valid padding which decreases it. The proof of Lemma~\ref{lem:diss_CNN} in Appendix~\ref{app:proof_lem_diss_CNN} relies on full padding, which over-approximates and thus includes the other cases as we argue in the following. The type of padding decides which finite excerpt $[\bN_1,\bN_2]$ of the infinite signal on $\bbN_0^d$ is passed on to the next layer.

Let $[\bN_1,\bN_2]$ define the excerpt that is used with same or no zero-padding. In case of full padding, the chosen excerpt involves all non-zero entries of $y^1[\cdot]$ and $y^2[\cdot]$ such that its value function is $V_{X}(y^1,y^2)$. Due to the positive-definiteness of $V_{X}$, its evaluation on a finite excerpt of the same signal yields the first inequality in
\begin{equation*}
    \sum_{\bi=\bN_1}^{\bN_2} y[\bi]^\top \widetilde{X} y[\bi]\leq V_{X}(y^1,y^2) \leq V_{X_{-}}(u^1,u^2),
\end{equation*}
and the second inequality, i.e., \eqref{eq:metricOpt_DP_relaxed}, is implied by \eqref{eq:cert_CNN}. This shows that \eqref{eq:cert_CNN} implies \eqref{eq:metricOpt_DP_relaxed} for unpadded and same-padded signals.  

\subsection{The fully connected layer}
If $\calL = \calF$ is a fully connected layer, then $\calD_{-} = \bbR^{c_{-}}$ and $\calD = \bbR^{c}$. In this case, we can understand $X_{-}$ and $X$ as matrices with $V_{X_{-}}(u^1,u^2) = (u^1 - u^2)^\top X_{-} (u^1 - u^2)$ and $V_{X}(y^1,y^2) = (y^1 - y^2)^\top X (y^1 - y^2)$, as mentioned before. We do not impose any further restrictions on $X$, $X_{-}$, i.\,e., $\calH_{\calF}^y = \bbR^{c \times c}$ and $\calH_{\calF}^u = \bbR^{c_{-} \times c_{-}}$. The following lemma describes \eqref{eq:metricOpt_DP_relaxed} as an LMI in a lossless manner.

\begin{lemma}\label{lem:diss_fullyConnected}
    Consider a fully connected layer $\calL=\calF$. With operators $X \in \calH_{\calF}^y$ and $X_{-} \in \calH_{\calF}^u$, a fully connected layer \eqref{eq:linearLayer} satisfies \eqref{eq:metricOpt_DP_relaxed} if and only if
    \begin{align} \label{eq:diss_fullyConnected}
        X_{-} - W^\top X W\succeq 0.
    \end{align}
\end{lemma}

LMI \eqref{eq:diss_fullyConnected} is a special case of the Lyapunov inequality \eqref{eq:cert_CNN} for $d = 0$, cmp. Remark~\ref{rem:correspondence}, and hence already proven. In this case, $\bA, \bB,\bC,\bP$ are empty matrices and $\bD$ corresponds to $W$. We denote the inequality \eqref{eq:diss_fullyConnected} by $\calG(X_{-},X,\nu) \succeq 0$, where $\nu = \lsb ~ \rsb$ (the empty matrix).

\subsection{The activation function layer}
If $\calL = \sigma$ is an activation function layer, then $\calD = \calD_{-} = \bbR^{c}$ and $\calD = \calD_{-} = \signals{d}{c}$ are both possible. Recall that the channel dimension $c_{-} = c$ and signal dimension $d_{-} = d$ stay the same for this layer type. In case $d>0$, we choose the operators $X$ and $X_{-}$ to be block diagonal and time-invariant, i.\,e., they satisfy \eqref{eq:restriction_conv}. The restriction of time-invariance is not needed for this layer type, which means that block-diagonal multipliers with varying blocks $\widetilde{X}\lsb\bi\rsb$ can also be used. However, we use the restriction \eqref{eq:restriction_conv}  for simplicity and computational tractability reasons. 

The most common activation functions such as ReLU, tanh, and sigmoid are slope-restricted, i.\,e., they satisfy the quadratic constraint \eqref{QC:slope_restriction} of the following lemma.

    \begin{lemma}[Slope-restriction \cite{fazlyab2023efficient,pauli2021training}] \label{lem:slope_restriction}
        Consider an activation function $\sigma:\bbR^c\to\bbR^c$ that is slope-restricted on $[0,1]$. For any  $\Lambda\in\bbD_{++}^c$, $\sigma$ satisfies
    \begin{equation}\label{QC:slope_restriction}
        \begin{bmatrix}
            x-y\\
            \sigma(x)-\sigma(y)
        \end{bmatrix}^\top
        \begin{bmatrix}
            0 & \Lambda\\
            \Lambda & -2\Lambda
        \end{bmatrix}
        \begin{bmatrix}
            x-y\\
            \sigma(x)-\sigma(y)
        \end{bmatrix}\geq 0, \,\forall x,y\in\bbR^c.
    \end{equation}
    \end{lemma}

Note that the published version of \cite{fazlyab2023efficient}, i.e., \cite{fazlyab2019efficient}, falsely used full matrix multipliers instead of diagonal $\Lambda$ which was later corrected by \cite{pauli2021training}. For slope-restricted activation functions, \eqref{eq:metricOpt_DP_relaxed} can be relaxed by an LMI as follows.

\begin{lemma}\label{lem:LMI_slope_restriction}
   Consider an activation function layer $\calL=\sigma$ that is slope-restricted on $[0,1]$. For some operators $X \in \calH_{\calC}^y$ and $X_{-} \in \calH_{\calC}^u$, this activation function layer satisfies \eqref{eq:metricOpt_DP_relaxed} if there exist $\Lambda\in \bbD^{c}_{++}$ such that
    \begin{align}
        \begin{bmatrix}
            \widetilde{X}_{-} & -\Lambda\\
            -\Lambda & 2\Lambda- \widetilde{X}
        \end{bmatrix} \succeq 0. \label{eq:diss_activation function}
    \end{align}
\end{lemma}

\begin{proof}
    For two arbitrary inputs $(u^1\lsb \bi \rsb ), (u^2\lsb \bi \rsb ) \in \signals{d}{c_{-}}$ with corresponding outputs $(y^1\lsb \bi \rsb ) , (y^2\lsb \bi \rsb )$, we left and right multiply \eqref{eq:diss_activation function} with $\begin{bmatrix}
        \Delta u[\bi]^\top & \Delta y[\bi]^\top
    \end{bmatrix}$,
    wherein $\Delta u[\bi]=u^1[\bi]-u^2[\bi]$ and $\Delta y[\bi]=y^1[\bi]-y^2[\bi]$, and its transpose, respectively, which yields
    \begin{align*}
        &\Delta u[\bi]^\top \widetilde{X}_{-} \Delta u[\bi] -\Delta y[\bi]^\top \widetilde{X} \Delta y[\bi]\\  &\geq 
        \begin{bmatrix}
            \Delta u[\bi]\\
            \Delta y[\bi]
        \end{bmatrix}^\top
        \begin{bmatrix}
            0 & \Lambda\\
            \Lambda & -2\Lambda
        \end{bmatrix}
        \begin{bmatrix}
            \Delta u[\bi]\\
            \Delta y[\bi]
        \end{bmatrix}.
    \end{align*}
    Subsequently, we sum over $\bi\in\bbN_0^d$ and obtain
    \begin{align*}
        &V_{X_{-}}(u^1,u^2) -
        V_{X}(y^1,y^2)\\  &\geq 
        \sum_{\bi\in\bbN_0^d}
        \begin{bmatrix}
            \Delta u[\bi]\\
            \Delta y[\bi]
        \end{bmatrix}^\top
        \begin{bmatrix}
            0 & \Lambda\\
            \Lambda & -2\Lambda
        \end{bmatrix}
        \begin{bmatrix}
            \Delta u[\bi]\\
            \Delta y[\bi]
        \end{bmatrix} \geq 0,
    \end{align*}
    wherein the last inequality follows from Lemma \ref{lem:slope_restriction}.
\end{proof}

Note that Lemma~\ref{lem:LMI_slope_restriction} also includes activation functions applied to vector spaces that occur subsequent to fully connected layers, where technically we need to infer $X \in \calH_{\calF}^y$ and $X_{-} \in \calH_{\calF}^u$ instead of $X \in \calH_{\calC}^y$ and $X_{-} \in \calH_{\calC}^u$. In fact, $X \in\calH_{\calF}^y$ and $X_{-} \in \calH_{\calF}^u$ are special cases of $X \in\calH_{\calC}^y$ and $X_{-} \in \calH_{\calC}^u$ for $d=0$, cmp. Remark~\ref{rem:correspondence}. We denote  the constraint \eqref{eq:diss_activation function} by $\calG(X_{-},X,\nu) \succeq 0$ where $\nu = \Lambda$.

Beside slope-restricted activations, another class of activation functions that has recently gained popularity are gradient norm preserving activations such as GroupSort and MaxMin \cite{anil2019sorting}. These activations are not applied element-wise but to a vector input $u[\bi]\in\bbR^{c}$ consisting of all preactivations at $\bi$. GroupSort separates the $c$ preactivations into $N$ groups each of size $n_g$, i.e., $c = Nn_g$, and then sorts these groups in ascending order. With the restriction
\begin{align*}
    X^\prime\lsb \bi, \bj \rsb = \begin{cases}
        \widetilde{X} \in\calT^c_{n_g} & \bi = \bj\\
        0 & \bi \neq \bj
    \end{cases}
\end{align*}
and an equivalent definition for $X^\prime_{-} \lsb \bi, \bj \rsb$, where
\begin{align*}
    \calT^c_{n_g} =\{T\in\bbS^c\mid T= \diag(\lambda)\otimes I_{n_g}+\diag(\gamma)\otimes \mathbf{1}_{n_g} \mathbf{1}_{n_g}^\top,~\\
    \lambda\in\bbR_+^{c/{n_g}},\gamma\in\bbR^{c/{n_g}} \},
\end{align*}
we can handle GroupSort activation functions using the following lemma \cite{pauli2024novel}. We denote these  structural constraints by $X \in \calH_{\sigma^{\mathrm{GS}}}^y$ and $X_{-} \in \calH_{\sigma^{\mathrm{GS}}}^u$. 

\begin{lemma}\label{lem:GroupSort}
    Consider a GroupSort activation function $\calL=\sigma^{\mathrm{GS}}$. For some operators $X\in\calH_{\sigma^{\mathrm{GS}}}^y$ and $X_{-}\in\calH_{\sigma^{\mathrm{GS}}}^u$, the GroupSort activation function satisfies \eqref{eq:metricOpt_DP_relaxed} if the matrices $\widetilde{X}$ and $\widetilde{X}_{-}$ satisfy $0\preceq \widetilde{X}\preceq \widetilde{X}_{-}$.
\end{lemma}

\begin{proof}
    The proof is deferred to Appendix~\ref{app:GroupSort}.
\end{proof}

\subsection{The pooling layer}
If $\calL = \calP$ is a pooling layer, then $\calD_{-} = \signals{d}{c_{-}}$ and $\calD = \signals{d}{c}$ for $d = d_{-}$ and $c = c_{-}$. The handling of pooling layers is very similar to the handling of activation function layers. As discussed in \cite{pauli2023lipschitza}, for both layer types there exist quadratic constraints exploiting their channel-wise Lipschitz properties, based on which we find LMI constraints for the respective layers.

Since pooling layers, i.e., subsampling layers, only make sense on the signal spaces $\signals{d}{c}$, we consider these signal spaces as domain and image spaces and restriction \eqref{eq:restriction_conv} on the operators $X$ and $X_{-}$. Note that theoretically, we could study this problem in the non-static, non-shift-invariant case. However, pooling layers will be concatenated with shift-invariant convolutional layers whose operators $X$ and $X_{-}$ are also chosen to be shift-invariant and static. The following lemma shows how we handle the dynamic programming inequality \eqref{eq:metricOpt_DP_relaxed} with average pooling layers.

\begin{lemma}
    For the average pooling layer, the dynamic programming inequality \eqref{eq:metricOpt_DP_relaxed} is satisfied if the matrices $\widetilde{X}$ and $\widetilde{X}_{-}$ satisfy the simple matrix inequality $0 \preceq \mu^2 \widetilde{X} \preceq \widetilde{X}_{-}$, $\mu$ being the Lipschitz constant of the average pooling layer.
\end{lemma}

Normally, we will set $\widetilde{X}_{-} = \mu^2\widetilde{X}$. For maximum pooling layers, we require an additional restriction, namely $\widetilde{X}=\diag(\lambda), \lambda\in\bbR^{c}$,  $\widetilde{X}_{-}=\diag(\lambda_{-}), \lambda_{-}\in\bbR^{c_{-}}$, yielding the next lemma.

\begin{lemma}
    For the maximum pooling layer, the dynamic programming inequality \eqref{eq:metricOpt_DP_relaxed} is satisfied if the matrices $\widetilde{X}$ and $\widetilde{X}_{k-1}$, that are parametrized as $\widetilde{X}=\diag(\lambda)$, $\widetilde{X}_{-}=\diag(\lambda_{-})$, satisfy $0\leq \mu^2 \lambda^j \leq \lambda_{-}^j$ for $j=1,\ldots,c$ with Lipschitz constant $\mu$ of the maximum pooling layer.
\end{lemma}

Again, we will normally require $\mu^2\lambda^j = \lambda_{-}^j$, $j=1,\dots,c$. It is common to choose the kernel size $\br+1$ and the stride $\bs$ of pooling layers to be the same. In that case the Lipschitz constant of a maximum pooling layer is $1$. We denote the restriction of the operators for maximum pooling layers, including the diagonality constraints, by $X \in \calH_{\calP^{\mathrm{max}}}^y$ and $X_{-} \in \calH_{\calP^{\mathrm{max}}}^u$. Furthermore, we denote by $\calG(X_{-},X,\nu) \succeq 0$ the respective constraint for these layers with $\nu = [~]$.

\subsection{Flattening operations}
In our setup, flattening operations have the role of reshaping tensor outputs from $\signals{d_{-}}{c_{-}}$ into vectors. In particular, they rearrange the output of a convolutional layer with $d_{-}>0$ as a vector to enable its use as an input for a fully connected layer. We have mentioned that, theoretically, a flattening operation could also map/project an element from $\signals{d_-}{c_{-}}$ to $\signals{d}{c}$, where $d_{-} > d$. However, we consider only the most relevant case of $d = 0$ in this section.

In this case, $\calR$ maps a patch $( u \lsb \bi \rsb \mid \bN_{1} \leq \bi \leq\bN_{2} )$
to the stacked vector of $u [\bi]$, i.\,e., $
    y = \begin{bmatrix}
        u \lsb \bN_{1} \rsb^\top & \cdots & u\lsb\bN_{2}\rsb^\top
    \end{bmatrix}^\top $.
Thus, we obtain the following lemma.

\begin{lemma}
    Consider a flattening operation $\calR : \signals{d_-}{c_-} \to \bbR^{c}$, with support $[\bN_{1},\bN_{2}]$ and $c = c_-|[\bN_{1},\bN_{2}]|$. The value function $V_{X_-}$ can be denoted as
    \begin{align*}
        V_{X_-}(u^1,u^2)=\!
        \sum_{\bi,\bj \in \bbN_0^{d}} (u^1\lsb \bi \rsb - u^2 \lsb \bi \rsb )^\top X^\prime_- \lsb \bi,\bj\rsb (u^1\lsb \bj \rsb - u^2 \lsb \bj \rsb ),
    \end{align*}
    whereas $V_{X}(y^1,y^2) = (y^1 - y^2 )^\top X (y^1 - y^2 )$. Then the dynamic programming inequality \eqref{eq:metricOpt_DP_relaxed} is satisfied if and only if
    \begin{align}
        \begin{bmatrix}
            X_-^\prime \lsb \bN_{1},\bN_{1} \rsb & \cdots & X^\prime_- \lsb \bN_{1},\bN_{2} \rsb\\
            \vdots & \ddots & \vdots\\
            X_-^\prime \lsb \bN_{2},\bN_{1} \rsb & \cdots & X^\prime_- \lsb \bN_{2},\bN_{2} \rsb
        \end{bmatrix}  
        \succeq X. \label{eq:flatteningDPIneq}
    \end{align}
\end{lemma}

The matrix inequality \eqref{eq:flatteningDPIneq} is an instance of $\calG(X_-,X,\nu) \succeq 0$ with $\nu=[~]$ and $\calH_{\calR}^y$ and $\calH_{\calR}^u$ technically impose no additional restrictions on $X_-$ and $X$ of the flattening operation. However, usually, the value function $V_{X}$ will be both static and time-invariant due to output restrictions on $X_-$ of the previous layer, e.g., $X_-\in\calH_{\calC}^{y}$, i.\,e., $X_-^\prime \lsb \bi,\bj\rsb = \widetilde{X}_-$ for $\bi = \bj$ and zero otherwise. 
In addition, we can require equality in \eqref{eq:flatteningDPIneq}, in which case $X = I_{|[\bN_{1},\bN_{2}]|} \otimes\widetilde{X}_-$ is a block-diagonal matrix with $|[\bN_{1},\bN_{2}]|$ copies of $\widetilde{X}_-$ on its diagonal. 

\subsection{The state space model layer}
The state space model layer $\calL=\calS$ is a generalization of the 1-D convolutional layer. The proof of Lemma~\ref{lem:diss_CNN} is independent of the structure of $A$, $B$, $C$, and $D$, which mark the difference between 1-D convolutional and state space model layers. Accordingly, we use \eqref{eq:cert_CNN} as a constraint $\calG(X_{-},X,\nu)\succeq 0$. %\eqref{eq:metricOpt_DP}.

\subsection{Subnetworks}\label{sec:subnetworks}
Up to now, we considered all building blocks of \eqref{eq:NN} as individual entities and require individual constraints \eqref{eq:metricOpt_DP_relaxed} for all these layers. However, for the implementation of \eqref{eq:metricOpt} and computational reasons, it is convenient to combine multiple layers as a subnetwork. We then include a constraint of type \eqref{eq:metricOpt_DP_relaxed} for the subnetwork. A typical concatenation is the combination of linear layers with the succeeding nonlinear activation functions, i.\,e., $\sigma\circ\calC$ for convolutional layers or $\sigma\circ\calF$ for fully connected layers.

\begin{lemma}\label{lem:conv+act}
    Consider the concatenation of a convolutional layer and an activation function layer, that is slope-restricted on $[0,1]$, $\calL=\sigma\circ\calC$.
    For some $X \in \calH_{\calC}^y$ and $X_{-} \in \calH_{\calC}^u$, the concatenation $\sigma\circ\calC$ satisfies \eqref{eq:metricOpt_DP_relaxed} if there exist symmetric matrices $P_m \in \bbS_{++}^{n_{m}}$, $\bP = \blkdiag (P_1,\ldots , P_d)$ and a diagonal matrix $\Lambda\in\bbD_{++}^{c}$ such that
    \begin{equation}\label{eq:cert_conv+act}
        \begin{bmatrix}
            \bP-\bA^\top \bP\bA  & -\bA^\top \bP\bB & -\bC^\top\Lambda\\
            -\bB^\top \bP\bA &  \widetilde{X}_{-}-\bB^\top \bP\bB  & -\bD^\top\Lambda\\
            -\Lambda \bC & -\Lambda \bD & 2\Lambda-\widetilde{X}
        \end{bmatrix}\succeq 0.
    \end{equation}
\end{lemma}

\begin{proof}
    The proof follows along the lines of the proof of Lemma \ref{lem:diss_CNN}, additionally using typical arguments from robust control \cite{scherer2000linear}. It can be found in Appendix \ref{sec:conv+act}.
\end{proof}

The condition \eqref{eq:cert_conv+act} is treated as an instance of $\calG(X,X_{-},\nu)$ with $\nu=(\bP,\Lambda)$. For an additional pooling layer, i.\,e. $\calP\circ\sigma\circ\calC$, we can extend Lemma \ref{lem:conv+act} easily by replacing $\widetilde{X}_-$ with $\frac{1}{\mu^2}\widetilde{X}_-$ and considering the output restriction $X\in\calH_{\calP_\mathrm{max}}^y$ in case a maximum pooling layer is added.

\begin{lemma}\label{lem:FC+act}
    Consider the concatenation of a fully connected layer and an activation function layer, that is slope-restricted on $[0,1]$, $\calL=\sigma\circ\calF$. For some $X \in \calH_{\calF}^y$ and $X_{-} \in \calH_{\calF}^u$, the concatenation $\sigma\circ\calF$ satisfies \eqref{eq:metricOpt_DP_relaxed} if there exists a diagonal matrix $\Lambda\in\bbD_{++}^{c}$ such that
    \begin{equation}\label{eq:cert_FC+act}
        \begin{bmatrix}
            X_{-}  & -W^\top\Lambda\\
            -\Lambda W & 2\Lambda-X
        \end{bmatrix}\succeq 0.
    \end{equation}
\end{lemma}

\begin{proof}
    We can view condition \eqref{eq:cert_FC+act} as a special case of \eqref{eq:cert_conv+act}, cmp. Remark~\ref{rem:correspondence}, and therefore, we refer to the proof of Lemma~\ref{lem:conv+act} in Appendix~\ref{sec:conv+act}.
\end{proof}

Note that we can also combine more layers, yielding larger and sparser LMIs but renouncing the decision variables $X$ at the transition between the layers. Extensions of Lemmas~\ref{lem:conv+act} and \ref{lem:FC+act} of this kind can be found in Appendix \ref{app:subnetworks}. If we combine all layers of a fully connected NN, we obtain the LMI originally proposed in \cite{fazlyab2019efficient}.

\begin{remark}
Throughout this subsection, we consider slope-restricted activations. However, all LMIs can also be formulated for GroupSort activations based on Lemma~\ref{lem:GroupSort} \cite{pauli2024novel}.
\end{remark}

\begin{remark}
    Beyond its use in Lipschitz analysis, SDP-based methods have been utilized for synthesis of Lipschitz-bounded NNs by parameterizing NN architectures such that they satisfy the underlying LMI conditions of an SDP by design \cite{revay2020lipschitz,revay2021recurrent,wang2023direct,pauli2023lipschitz}. In this sense, the derived LMI conditions of this section build a basis for such a construction \cite{pauli2024lipkernel}.
\end{remark}

\subsection{Residual layers and skip connections}\label{sec:ResLayer}
In deep learning, NN structures that include skip connections, called residual NNs or ResNets, have proven to avoid vanishing and exploding gradients \cite{he2016deep}. We define such residual layers as follows
\begin{equation}\label{eq:ResNet_layer_general}
    y=u+\calM(u),
\end{equation}
where $\calM(u)$ is a feedforward NN \eqref{eq:NN} of arbitrary length and $u \in \calD_{-}$ and $y \in \calD$. We in addition require $\calD = \calD_{-}$ as well as $\calM : \calD_{-} \to \calD_{-}$. For example, a ResNet layer that skips a fully connected network with one hidden layer reads
\begin{equation}\label{eq:ResNet_layer}
   y=u+W_2\sigma(W_1 u+b_1)+b_2.
\end{equation}
with $W_1\in\bbR^{n_{v}\times c_{-}}$, $W_2\in\bbR^{c\times n_{v}}$, $b_1\in\bbR^{n_v}$, $b_2\in\bbR^{c}$, $n_{v}$ being the dimension of $v\coloneqq\sigma(W_1 u+b_1)$.

In the following lemma, we describe how the simple skip connection \eqref{eq:ResNet_layer} leads to an LMI relaxation for \eqref{eq:metricOpt_DP_relaxed}. %More general skip connections can be treated with the same arguments; see Appendix~\ref{app:ResNets}.
\begin{lemma}\label{lem:ResNet_1_layer}
    Consider a residual layer \eqref{eq:ResNet_layer} with activation functions that are slope-restricted on $[0,1]$. For some $X \in \calH_{\calF}^y$ and $X_{-} \in \calH_{\calF}^u$, the ResNet layer \eqref{eq:ResNet_layer} satisfies \eqref{eq:metricOpt_DP_relaxed} if there exist $\Lambda\in\bbD_{++}^{n_{v}}$ such that
    \begin{equation}\label{eq:cert_ResNet}
        \begin{bmatrix}
            X_{-}-X       & -W_1^\top\Lambda-XW_2\\
            -\Lambda W_1-W_2^\top X & 2\Lambda-W_2^\top X W_2
        \end{bmatrix}\succeq 0.
    \end{equation}
\end{lemma}

\begin{proof}
    A proof can be found in Appendix~\ref{app:ResNet_cert}. It follows the same arguments as the proof of \cite[Theorem 4]{araujoICLR2023}.
\end{proof}

\begin{remark}
In a similar fashion, we can pose LMIs for ResNet layers which skip more than two layers or include convolutional layers. Due to space limitations, we do not discuss these cases.
\end{remark}

\section{Analysis of the conservatism}\label{sec:conservatism}

With the exact dynamic programming recursion \eqref{eq:DP_recursion} it is (theoretically) possible to compute the exact Lipschitz constant of an NN using \eqref{eq:metricOpt}. For reasons of computational tractability, however, we propose the relaxation \eqref{eq:finalSDP}. For the derivation of this SDP, several relaxation steps were made resulting in the following sources of conservatism.
\begin{enumerate}%[label=(\roman*)]
    \item[(i)] Quadratic constraints that over-approximate nonlinearities.
    \item[(ii)] Layer specific restrictions $X_{k-1} \in \calH_{\calL_k}^u$ and $X_{k} \in \calH_{\calL_k}^y$.
    \item[(iii)] Cut-off errors caused by handling convolutional layers as mappings on infinite-dimensional sequence spaces, whereas in reality only finitely supported image signals are processed.
    \item[(iv)] Quadratic value functions \eqref{eq:quadratic_value_function}.
\end{enumerate}
Number (i) enables the incorporation of nonlinear activation functions and pooling layers into an SDP-based analysis \cite{fazlyab2020safety} and therefore this relaxation is essential for the tractability of the problem. Moreover, (iii) can be viewed as a special case of (ii), since considering general operators $X_k$ instead of space-shift invariant operators would resolve this issue at the price of significantly larger LMIs. This shows that there is a trade-off between scalability and accuracy.

The quadratic value functions (iv) allow us to leverage the feedforward interconnection structure of the NN, yielding layer-wise LMIs and resulting in computational advantages and superior scalability compared to \cite{fazlyab2019efficient,gramlich2023convolutional}. Next, we justify why layer-wise splitting is \emph{not} more restrictive than the use of one large and sparse LMI constraint as done in \cite{fazlyab2019efficient,gramlich2023convolutional}. To this end, we analyze a fully connected NN of the form
\begin{align*}
    \mathrm{FNN}_\theta = \calF_l \circ \sigma \circ \cdots \circ \sigma \circ \calF_1,
\end{align*}
also considered in LipSDP \cite{fazlyab2019efficient}. As suggested in \cite{pauli2023lipschitza}, we can use the semidefinite constraint 
\begin{align}\label{eq:diss_coupling}
{\small
	\begin{bmatrix}
		Q_l - I & S_l\\
		S_l^\top & R_l + Q_{l-1} & S_{l-1}\\
		& S_{l-1}^\top & R_{l-1} + Q_{l-2} & \ddots\\
		  & & \ddots & \ddots & S_1\\
		  & & & S_1^\top & R_1 + \gamma^2 I
	\end{bmatrix} \preceq 0,}
\end{align}
where
\begin{align}
    \begin{bmatrix} \label{eq:dissipativityCond}
        -2\Lambda_k & \Lambda_k W_k\\
        W_k^\top \Lambda_k & 0
    \end{bmatrix} \preceq
    \begin{bmatrix}
        Q_k & S_k\\
        S_k^\top & R_k
    \end{bmatrix}, \quad k = 1,\dots,l
\end{align}
is used instead of \eqref{eq:cert_FC+act} as layer-wise LMI constraints for fully connected layers for Lipschitz constant estimation. To recover LipSDP \cite{fazlyab2019efficient} simply replace the conic inequality in \eqref{eq:dissipativityCond} with an equality.

The following theorem implies that it poses no restriction to parameterize the dissipativity blocks as $\begin{bsmallmatrix}
        Q_k & S_k\\
        S_k^\top & R_k
    \end{bsmallmatrix} = \begin{bsmallmatrix}
        X_k & 0\\
        0 & -X_{k-1}
    \end{bsmallmatrix}$, where $X_l = I$ and $X_0 = \gamma^2 I$, as done in this work.
\begin{theorem}
    \label{thm:noMoreConservativeThanQSR}
    Assume that the matrix inequality \eqref{eq:diss_coupling} is satisfied. Then there exists a sequence of matrices $X_0,\ldots,X_l$ such that $X_0 = \gamma^2 I$, $X_l = I$ and
    \begin{align}\label{eq:LMI_Thm1}
        \begin{bmatrix}
        Q_k & S_k\\
        S_k^\top & R_k
    \end{bmatrix} \preceq \begin{bmatrix}
        X_k & 0\\
        0 & -X_{k-1}
    \end{bmatrix},\quad k=1,\dots l.
    \end{align}
\end{theorem}
\begin{proof}
    See Appendix \ref{app:proof_proposition}.
\end{proof}

It follows from Theorem \ref{thm:noMoreConservativeThanQSR} that the LMIs \eqref{eq:diss_coupling}, \eqref{eq:dissipativityCond} are equivalent to \eqref{eq:cert_FC+act}, $k=1,\dots,l$, $X_0=\gamma^2I$, $\Lambda_l=I$, $X_k=I$, i.e, that layer-wise splitting does not introduce conservatism. Consequently, 
%Theorem \ref{thm:noMoreConservativeThanQSR} shows that 
the optimal value of $\gamma$ found solving 
\begin{equation}\label{eq:SDP1}
    \min_{\gamma^2,\Lambda,Q,S,R}~\gamma^2\quad\text{s.\,t.}\quad \eqref{eq:diss_coupling},\eqref{eq:dissipativityCond},
\end{equation}
where $\Lambda=(\Lambda_1,\dots,\Lambda_l)$, $Q=(Q_1,\dots,Q_l)$, $S=(S_1,\dots,S_l)$ and $R=(R_1,\dots,R_l)$, and the optimal value of $\gamma$ from 
\begin{equation}\label{eq:SDP2}
    \min_{\gamma^2,\Lambda,X}~\gamma^2~\text{s.\,t.}~ \eqref{eq:cert_FC+act},~k=1,\dots,l,~X_0=\gamma^2I,X_l=I,\Lambda_l=I
\end{equation}
where $\Lambda=(\Lambda_1,\dots,\Lambda_l)$, $X=(X_1,\dots,X_{l-1})$ using the parameterization of $Q_k$, $S_k$, $R_k$ via $X_k$, are equivalent. We note that \eqref{eq:SDP2} is an instance of \eqref{eq:finalSDP} for a fully connected NN. Another consequence of the result in Theorem~\ref{thm:noMoreConservativeThanQSR} is that our approach of choosing the sequence of matrices $(X_k)$ that satisfy \eqref{eq:LMI_Thm1} is not more conservative than using the large and sparse constraint in LipSDP. The relation of \eqref{eq:finalSDP} to \cite{gramlich2023convolutional}, that includes convolutional layers, can be shown in a similar fashion.

\begin{remark}\label{rem:recasting}
    Convolutional layers can be recast as fully connected layers \cite{Goodfellow-et-al-2016} and the experiments in \cite{gramlich2023convolutional} show that this recasting can reduce conservatism in comparison to \eqref{eq:finalSDP}. While (i) also causes conservatism in LipSDP \cite{fazlyab2019efficient}, through the recasting of convolutional layers as fully connected layers, the relaxations (ii) and (iii) can be avoided in LipSDP. However, as it also becomes apparent in \cite{gramlich2023convolutional}, this relaxation has a high computational cost.
\end{remark}

\begin{remark}
    For the special case of a fully connected NN our proposed layer-wise LMI constraints \eqref{eq:cert_FC+act} correspond to the decomposition of the LMI in LipSDP by chordal sparsity \cite{xue2022chordal}, also yielding a set of LMI constraints that are equivalent to LipSDP \cite{fazlyab2019efficient}.
\end{remark}
\begin{remark}
    The result of Theorem \ref{thm:noMoreConservativeThanQSR} can be interpreted as the statement that for a series interconnection of QSR-dissipative mappings, it suffices to consider supply rates $ s(u_k^1 - u_k^2,y_k^1 - y_k^2)$ of the form
    \begin{align*}
        \left\langle
        \begin{bmatrix}
            u_k^1 - u_k^2\\
            y_k^1 - y_k^2
        \end{bmatrix}
        ,
        \begin{bmatrix}
            X_k & 0\\
            0 & -X_{k-1}
        \end{bmatrix}
        \begin{bmatrix}
            u_k^1 - u_k^2\\
            y_k^1 - y_k^2
        \end{bmatrix}
        \right\rangle_2.
    \end{align*}
\end{remark}

To summarize, the SDP \eqref{eq:finalSDP} exploits the structure of NNs in two ways. Firstly, it exploits the concatenation structure of NNs to generate $l$ small LMI constraints instead of one large and sparse constraint, and, secondly, it utilizes the fact that convolutional layers and state space model layers are dynamical systems. This gives \eqref{eq:finalSDP} one advantage over \cite{gramlich2023convolutional}, where only the dynamical system nature of convolutional layers is exploited and two advantages over LipSDP \cite{fazlyab2019efficient} in terms of scalability.

\section{Experiments}\label{sec:experiments}
In this section, we compare our newly proposed method GLipSDP (general LipSDP) to state-of-the-art SDP-based methods for Lipschitz constant estimation w.r.t. the $\ell_2$ norm, namely LipSDP \cite{fazlyab2019efficient} and CLipSDP (convolutional LipSDP) \cite{gramlich2023convolutional} and to a recent loop transformation based method LipLT \cite{fazlyab2024certified}. In addition, we compute a trivial matrix norm product bound (MP), the product of the spectral norms of the weights \cite{szegedy2013intriguing}. We first compare the accuracy and scalability of the methods on fully connected and fully convolutional networks, respectively, in Subsections~\ref{sec:Exp_fully_connected} and \ref{sec:Exp_fully_convolutional} and discuss the influences of the sources of conservatism that were introduced in \ref{sec:conservatism}. In Subsection~\ref{sec:Exp_different_architectures}, we then apply GLipSDP to different architectures for CNNs to illustrate its versatility.  The code is written in a modular fashion such that it can be applied easily to any NN architecture involving layers considered in this paper. All computations are carried out on a standard i7 note book with 32 GB RAM using Yalmip \cite{yalmip04} with the solver Mosek \cite{mosek} in Matlab for LipSDP, GLipSDP and CLipSDP and Python and Pytorch for LipLT and MP\footnote{We provide our code at \url{https://github.com/ppauli/GLipSDP}.}.

\subsection{Scalability on fully connected networks}\label{sec:Exp_fully_connected}
\begin{figure}
    \centering
\definecolor{mycolor6}{rgb}{0.00000,0.44700,0.74100}%
\definecolor{mycolor2}{rgb}{0.85000,0.32500,0.09800}%
\definecolor{mycolor3}{rgb}{0.92900,0.69400,0.12500}%
\definecolor{mycolor4}{rgb}{0.49400,0.18400,0.55600}%
\definecolor{mycolor5}{rgb}{0.46600,0.67400,0.18800}%
\definecolor{mycolor1}{rgb}{0.30100,0.74500,0.93300}%
\definecolor{mycolor7}{rgb}{0.63500,0.07800,0.18400}%
\begin{tikzpicture}
\begin{axis}[%
width=0.8in,
height=0.44in,
at={(0in,0in)},
scale only axis,
bar shift auto,
xmin=0.509090909090909,
xmax=3.49090909090909,
xtick={1,2,3,4},
xticklabels={{2},{4},{8},{16}},
ytick={0,0.5,1},
xlabel style={font=\color{white!15!black}},
xlabel={Depth},
ymin=0,
ymax=1,
ylabel style={font=\color{white!15!black}, yshift=-0.3cm},
ylabel={$\gamma$},
%ylabel near ticks,
ybar=0pt,
axis background/.style={fill=white},
legend style={at={(0.48,2.7)}, anchor=north, legend columns=1, legend cell align=left, align=left, draw=none}
]
\addplot[ybar, bar width=3, fill=mycolor1, draw=black, area legend] table[row sep=crcr] {%
1	0.648476712405906\\
2	0.289742405302762\\
3	0.0527767046155736\\
};
\addplot[forget plot, color=white!15!black] table[row sep=crcr] {%
0.509090909090909	0\\
3.49090909090909	0\\
};
\addlegendentry{LipSDP}

\addplot[ybar, bar width=3, fill=mycolor2, draw=black, area legend] table[row sep=crcr] {%
1	0.648476715567287\\
2	0.289742414628371\\
3	0.0527769652946622\\
};
\addplot[forget plot, color=white!15!black] table[row sep=crcr] {%
0.509090909090909	0\\
3.49090909090909	0\\
};
\addlegendentry{GLipSDP}

\addplot[ybar, bar width=3, fill=mycolor3, draw=black, area legend] table[row sep=crcr] {%
1	0.813518783137531\\
2	0.534794431138905\\
3	0.21197974071739\\
};
\addplot[forget plot, color=white!15!black] table[row sep=crcr] {%
0.509090909090909	0\\
3.49090909090909	0\\
};
\addlegendentry{LipLT}

\addplot[ybar, bar width=3, fill=mycolor4, draw=black, area legend] table[row sep=crcr] {%
1	1\\
2	1\\
3	1\\
4   1\\
};
\addplot[forget plot, color=white!15!black] table[row sep=crcr] {%
0.509090909090909	0\\
3.49090909090909	0\\
};
\addlegendentry{MP}
\end{axis}
\node[align=center,font=\bfseries] at (0.65in,0.53in) {$c=32$};

\begin{axis}[%
width=2.12in,
height=1.5in, % Adjusted height
at={(1.4in,0in)},
xmode=log,
xmin=2,
xmax=64,
xtick={ 2,  4,  8,  16, 32, 64, 128},
xticklabels={ 2,  4,  8,  16, 32, 64, 128},
ylabel={Computation Time (sec)},
xlabel={Depth},
ymode=log,
ymin=0.01,
ymax=3000,
yminorticks=true,
ylabel style={font=\color{white!15!black}, yshift=-0.05cm},
axis line style={-},
tick align=outside,
tick pos=left,
legend style={at={(0.43,1.3)}, anchor=north, legend columns=3, legend cell align=left, draw=none},
every axis plot/.append style={very thick} % Adjusted to make lines thicker
]
\addplot [color=black, forget plot]
  table[row sep=crcr]{%
2	0.0617122\\
};
\label{solid}
\addplot [color=mycolor5, mark=square, line width=1pt]
  table[row sep=crcr]{%
2	0.0617122\\
4	0.0401969\\
8	0.0715163\\
16	0.204158\\
32	0.3969887\\
64	0.9105674\\
};
\addlegendentry{$c=16$}

\addplot [color=mycolor6, mark=triangle, line width=1pt]
  table[row sep=crcr]{%
2	0.0588946\\
4	0.315951\\
8	1.0585204\\
16	3.1369859\\
32	6.7633474\\
64	10.3227533\\
};
\addlegendentry{$c=32$}

\addplot [color=mycolor7, mark = diamond, line width=1pt]
  table[row sep=crcr]{%
2	1.0681704\\
4	6.8938716\\
8	20.7272597\\
16	69.5903507\\
32	126.6435941\\
64	274.3208863\\
};
\addlegendentry{$c=64$}
\addplot [color=black, dashed, forget plot, line width=1pt]
  table[row sep=crcr]{%
2	0.0139982\\
};\label{dashed}

%\addplot [color=mycolor4, mark = o, line width=1pt]
%  table[row sep=crcr]{%
%4	8.4603274\\
%8	23.8803126\\
%16	73.5247766\\
%32	148.0413328\\
%64	0\\
%128	0\\
%};
%\addlegendentry{$c=64$}

\addplot [color=mycolor5, mark=square, mark options={solid}, dashed, forget plot, line width=1pt]
table[row sep=crcr]{%
2	0.0139982\\
4	0.0311817\\
8	0.1324681\\
16	0.675343\\
32	3.8403408\\
64	25.6006774\\
};
\addplot [color=mycolor6, mark = triangle, mark options={solid}, dashed, forget plot, line width=1pt]  
table[row sep=crcr]{%
2	0.0292963\\
4	0.1149181\\
8	0.7382071\\
16	6.2245605\\
32	27.7334686\\
64	243.3759809\\
};
\addplot [color=mycolor7, mark = diamond, mark options={solid}, dashed, forget plot, line width=1pt]  
table[row sep=crcr]{%
2	0.1474984\\
4	1.0813596\\
8	5.0801539\\
16	40.7650554\\
32	298.4798369\\
64	0\\
};
%\addplot [color=mycolor4, mark = o, mark options={solid}, dashed, line width=1pt]
%  table[row sep=crcr]{%
%4	1.1838669\\
%8	6.2093277\\
%16	47.0452086\\
%32	318.029827\\
%64	0\\
%128	0\\
%};
\end{axis}

\end{tikzpicture}%
    \caption{Lipschitz bounds $\gamma$ using LipSDP, GLipSDP, LipLT and the matrix product bound (MP) on fully connected NNs with depths $d=\{2,4,8\}$ and hidden layer size $c=32$ (LipSDP/GLipSDP bounds are close to zero for deeper NNs) (left). Computation times for fully connected NNs with depths $d=\{2,4,8,16,32,64\}$ and channel sizes $c=\{16,32,64\}$ for GLipSDP (\ref{solid}) and LipSDP (\ref{dashed}) (right).}
    \label{fig:Lipschitz_bounds_FC}
\end{figure}
In this subsection, we consider fully connected NNs $\calF\circ\sigma\circ\calF\cdots\sigma\circ\calF$ of different widths $c=\{16,32,64\}$ and depths $d=\{2,4,8,16,32,64\}$ and compare GLipSDP (ours), LipSDP, LipLT and MP in terms of accuracy and the SDP-based variants GLipSDP and LipSDP in terms of  computation times. To do so we randomly generated weights that we subsequently normalized such that the MP bound is 1 and Fig.~\ref{fig:Lipschitz_bounds_FC} (left) shows the bounds computed using LipSDP, GLipSDP and LipLT for NNs of different depths $d=\{ 2, 4, 8 \}$. We observe that LipSDP and GLipSDP yield the same bounds, as expected from Theorem~\ref{thm:noMoreConservativeThanQSR}, which are lower than the ones obtained by LipLT. As the depth increases the gap between MP and the other bounds becomes larger. In Fig. \ref{fig:Lipschitz_bounds_FC} (right), we now compare the computation times of LipSDP and GLipSDP. While LipSDP relies on one large and sparse LMI constraint, GLipSDP considers one dense LMI per layer. We notice that for shallow NNs of depth of up to 8 layers, LipSDP is faster than layer-wise GLipSDP. However, with increasing depths the layer-wise splitting of GLipSDP is advantageous in terms of computation times. This trade-off is due to the fact that the splitting also introduces additional decision variables $X$ at the layer transitions. A direct consequence of this phenomenon is that it may be computationally advantageous to combine some fully connected layers as subnetworks.

Next, we only consider the 32-layer fully connected networks with 32 and 64 neurons ($d=32$, $c=\{32,64\}$) and we apply GLipSDP but vary the number of layers combined in subnetworks, cmp. Section~\ref{sec:subnetworks}. More specifically, we compute an upper Lipschitz bound using layer-wise LMI constraints, i.e., 32 LMI constraints, and then combine 2, 4, 8, 16, 32 layers to form subnetworks, then applying GLipSDP with 16, 8, 4, 2, and 1 LMI constraints instead of 32. It is an immediate result of Theorem \ref{thm:noMoreConservativeThanQSR} that GLipSDP yields the same Lipschitz bounds for all subnetwork configurations, yet it requires different computation times that are shown in Fig.~\ref{fig:Comp_times_subnetworks} (left). The least computation time is achieved using 8 subnetworks of depth 4 in the 32-neuron case and 4 subnetworks of depth 8 in the 64-neuron variant, again illustrating the trade-off due to smaller LMI constraints on the one hand and more decision variables on the other hand.

\subsection{Scalability on fully convolutional networks}\label{sec:Exp_fully_convolutional}
\begin{figure}
    \centering
\definecolor{mycolor6}{rgb}{0.00000,0.44700,0.74100}%
\definecolor{mycolor2}{rgb}{0.85000,0.32500,0.09800}%
\definecolor{mycolor3}{rgb}{0.92900,0.69400,0.12500}%
\definecolor{mycolor4}{rgb}{0.49400,0.18400,0.55600}%
\definecolor{mycolor5}{rgb}{0.46600,0.67400,0.18800}%
\definecolor{mycolor1}{rgb}{0.30100,0.74500,0.93300}%
\definecolor{mycolor7}{rgb}{0.63500,0.07800,0.18400}%
\begin{tikzpicture}
\begin{axis}[%
width=0.85in,
height=0.44in,
at={(0in,0in)},
scale only axis,
bar shift auto,
xmin=0.511111111111111,
xmax=4.48888888888889,
xtick={1,2,3,4},
xticklabels={{2},{4},{8},{16}},
ytick={0,5,10,15},
xlabel style={font=\color{white!15!black}},
xlabel={Depth},
ymin=0,
ymax=12,
ylabel style={font=\color{white!15!black}, yshift=-0.5cm},
ylabel={$\log(\gamma)$},
%ylabel near ticks,
ybar=0pt,
axis background/.style={fill=white},
legend style={at={(0.48,2.7)}, anchor=north, legend columns=1, legend cell align=left, align=left, draw=none}
]
\addplot[ybar, bar width=3, fill=mycolor1, draw=black, area legend] table[row sep=crcr] {%
1	1.39514120827907\\
2	3.00339768882389\\
3	4.2209251138703\\
4	8.37512291061585\\
};
\addplot[forget plot, color=white!15!black] table[row sep=crcr] {%
0.509090909090909	0\\
4.49090909090909	0\\
};
\addlegendentry{CLipSDP}

\addplot[forget plot, color=white!15!black] table[row sep=crcr] {%
0.509090909090909	0\\
4.49090909090909	0\\
};
\addplot[ybar, bar width=3, fill=mycolor2, draw=black, area legend] table[row sep=crcr] {%
1	1.39514120827907\\
2	3.00339768882389\\
3	4.2209251138703\\
4	8.37512291061585\\
};
\addplot[forget plot, color=white!15!black] table[row sep=crcr] {%
0.509090909090909	0\\
4.49090909090909	0\\
};
\addlegendentry{GLipSDP}

\addplot[ybar, bar width=3, fill=mycolor3, draw=black, area legend] table[row sep=crcr] {%
1	1.44590670605103\\
2	3.17491975760197\\
3	4.82012994120741\\
4	9.96383906882735\\
};
\addplot[forget plot, color=white!15!black] table[row sep=crcr] {%
0.509090909090909	0\\
4.49090909090909	0\\
};
\addlegendentry{LipLT}

\addplot[ybar, bar width=3, fill=mycolor4, draw=black, area legend] table[row sep=crcr] {%
1	1.50983854721135\\
2	3.40625125596483\\
3	5.50744312609404\\
4	11.4669443517204\\
};
\addplot[forget plot, color=white!15!black] table[row sep=crcr] {%
0.509090909090909	0\\
4.49090909090909	0\\
};
\addlegendentry{MP}
\end{axis}
\node[align=center,font=\bfseries] at (0.65in,0.53in) {$c=16$};

\begin{axis}[%
width=2.15in,
height=1.5in, % Adjusted height
at={(1.5in,0in)},
xmode=log,
xmin=2,
xmax=16,
xtick={ 2,  4,  8, 16},
xticklabels={ 2,  4,  8, 16},
ylabel={Computation Time (sec)},
xlabel={Depth},
ymode=log,
ymin=0.888861,
ymax=30000,
yminorticks=true,
ylabel style={font=\color{white!15!black}, yshift=-0.1cm},
axis line style={-},
tick align=outside,
tick pos=left,
legend style={at={(0.42,1.3)}, anchor=north, legend columns=3, legend cell align=left, draw=none},
every axis plot/.append style={very thick} % Adjusted to make lines thicker
]
\addplot [color=black, forget plot]
  table[row sep=crcr]{%
2	0.888861\\
};
\label{solid}
\addplot [color=mycolor5, mark=square, line width=1pt]
  table[row sep=crcr]{%
2	0.888861\\
4	1.54314\\
8	3.14382\\
16	7.600219\\
};
\addlegendentry{$c=8$}

\addplot [color=mycolor6, mark=triangle, line width=1pt]
  table[row sep=crcr]{%
2	10.841696\\
4	34.772299\\
8	75.496793\\
16	171.258854\\
};
\addlegendentry{$c=16$}

\addplot [color=mycolor7, mark = diamond, line width=1pt]
  table[row sep=crcr]{%
2	431.403134\\
4	1769.020011\\
8	4951.040923\\
16	0\\
};
\addlegendentry{$c=32$}

\addplot [color=black, dashed, forget plot, line width=1pt]
  table[row sep=crcr]{%
2	1.102187\\
};\label{dashed}
\addplot [color=mycolor5, mark=square, mark options={solid}, dashed, forget plot, line width=1pt]
  table[row sep=crcr]{%
2	1.102187\\
4	5.652111\\
8	29.704354\\
16	137.131204\\
};
\addplot [color=mycolor6, mark = triangle, mark options={solid}, dashed, forget plot, line width=1pt]
  table[row sep=crcr]{%
2	19.608339\\
4	169.656385\\
8	920.229262\\
16	3397.828552\\
};
\addplot [color=mycolor7, mark = diamond, mark options={solid}, dashed, forget plot, line width=1pt]
  table[row sep=crcr]{%
2	751.589044\\
4	7747.938968\\
8	0\\
16	0\\
};
%\addplot [color=mycolor1, mark=square, only marks, mark=o, mark options={solid, mycolor1}, forget plot, line width=1pt]
%  table[row sep=crcr]{%
%2	4573.160195\\
%};\label{o}
\end{axis}

\end{tikzpicture}%
    \caption{Lipschitz bounds $\gamma$ using CLipSDP, GLipSDP, LipLT and the matrix product bound (MP) on fully convolutional NNs with depths $d=\{2,4,8,16\}$ and channel sizes $c=16$ (left). Computation times for fully convolutional networks with depths $d=\{2,4,8,16\}$ and channel sizes $c=\{8,16,32\}$ for GLipSDP (\ref{solid}) and CLipSDP (\ref{dashed}) (right).} %\ref{o} indicates the computation time using LipSDP. For larger networks, LipSDP runs into memory issues (right).}
    \label{fig:Lipschitz_bounds}
\end{figure}

\begin{figure}
    \centering
    \definecolor{mycolor6}{rgb}{0.00000,0.44700,0.74100}%
\definecolor{mycolor2}{rgb}{0.85000,0.32500,0.09800}%
\definecolor{mycolor3}{rgb}{0.92900,0.69400,0.12500}%
\definecolor{mycolor4}{rgb}{0.49400,0.18400,0.55600}%
\definecolor{mycolor5}{rgb}{0.46600,0.67400,0.18800}%
\definecolor{mycolor1}{rgb}{0.30100,0.74500,0.93300}%
\definecolor{mycolor7}{rgb}{0.63500,0.07800,0.18400}%

\begin{tikzpicture}

\begin{axis}[%
width=1.65in,
height=1.6in, % Adjusted height
xmin=1,
xmax=6,
xtick={1,2,3,4,5,6},
xticklabels={{32},{16},{8},{4},{2},{1}},
ylabel={Computation Time (sec)},
xlabel={\# of subnetworks},
ymode=log,
ymin=2,
ymax=300,
yminorticks=true,
axis line style={-},
tick align=outside,
tick pos=left,
legend style={at={(0.35,1.24)}, anchor=north, legend columns=2, legend cell align=left, draw=none},
every axis plot/.append style={very thick} % Adjusted to make lines thicker
]
\addplot [color=mycolor6, mark=triangle, mark options={solid}]
  table[row sep=crcr]{%
1	6.2205022\\
2	3.8018979\\
3	3.084915\\
4	3.5077172\\
5	7.3516327\\
6   22.9424708\\
};
\addlegendentry{$c=32$}
\addplot [color=mycolor7, mark=diamond, mark options={solid}]
  table[row sep=crcr]{%
1	124.8171004\\
2	84.5072613\\
3	57.0682013\\
4	55.2163157\\
5	87.3525793\\
6   259.6408457\\
};
\addlegendentry{$c=64$}
\end{axis}

\begin{axis}[%
width=1.65in,
height=1.6in, % Adjusted height
at={(1.7in,0in)},
xmin=1,
xmax=5,
xtick={1,2,3,4,5},
xticklabels={{16},{8},{4},{2},{1}},
ylabel={Computation Time (sec)},
xlabel={\# of subnetworks},
ymode=log,
ymin=5,
ymax=5000,
yminorticks=true,
axis line style={-},
tick align=outside,
tick pos=left,
legend style={at={(0.35,1.24)}, anchor=north, legend columns=2, legend cell align=left, draw=none},
every axis plot/.append style={very thick} % Adjusted to make lines thicker
]
\addplot [color=mycolor5, mark=square, mark options={solid}]
  table[row sep=crcr]{%
1	7.414865\\
2	17.039385\\
3	27.089757\\
4	51.231301\\
5	114.759862\\
};
\addlegendentry{$c=8$}
\addplot [color=mycolor6, mark=triangle, mark options={solid}]
  table[row sep=crcr]{%
1	215.886687\\
2	484.162856\\
3	997.515512\\
4	2258.09747\\
5	3673.156972\\
};
\addlegendentry{$c=16$}
\end{axis}

\end{tikzpicture}%
    \caption{Computation times for GLipSDP using 32, 16, 8, 4, 2, 1 subnetworks for a 32-layer fully connected network with 32 and 64 neurons (left). Computation times for GLipSDP using 16, 8, 4, 2, 1 subnetworks for a 16-layer fully convolutional network with 8 and 16 channels (right). The resulting Lipschitz bound is the same for all computations.}
    \label{fig:Comp_times_subnetworks}
\end{figure}

Next, we consider fully convolutional NNs. In particular, we train CNNs with backbones of depths $d=\{2,4,8,16\}$ and channel sizes $c=\{8,16,32\}$ on the MNIST dataset \cite{mnist}. We then analyze the fully convolutional backbones of these NNs, i.e., a subnetwork $\sigma\circ\calC\cdots\sigma\circ\calC$  which only consists of convolutional layers. The input size to the backbone is $14\times 14$ and is kept constant throughout all backbone layers.

Fig.~\ref{fig:Lipschitz_bounds} (left) shows the Lipschitz bounds obtained using GLipSDP (ours), CLipSDP, and MP. We refrain from a comparison to LipSDP as the underlying SDP runs into memory issues for all networks except the smallest one ($c=8$, $d=2$). We first note that GLipSDP and CLipSDP produce the same bounds which are tighter than the ones found with LipLT and significantly tighter than the trivial matrix product bound MP. In Fig.~\ref{fig:Lipschitz_bounds} (right) we compare the computation times of GLipSDP and CLipSDP. What distinguishes our method GLipSDP from CLipSDP is that we consider layer-wise LMIs rather than one large and sparse LMI constraint. In the case of fully convolutional layers, we observe that the deeper the network the larger the computational advantage of using GLipSDP. 

In Fig.~\ref{fig:Comp_times_subnetworks} (right), we apply GLipSDP defining subnetworks of different depths to analyze the 16-layer fully convolutional network with channel sizes 8 and 16 ($d=16$, $c=\{8,16\}$). In this experiment, we see that it is computationally advantageous to use multiple smaller LMI constraints, i.e., exploit the layer-by-layer structure of the network, in fully convolutional NNs. It follows from Theorem \ref{thm:noMoreConservativeThanQSR} that the obtained Lipschitz bounds are the same for all chosen splits.

\subsection{Convolutional neural networks for image classification}\label{sec:Exp_different_architectures}
Next, we compute upper bounds on the Lipschitz constant for typical CNN architectures, including  LeNet-5 \cite{lecun1998gradient}, the NNs used in \cite{leino2021globally}, and 18-layer residual NNs. These experiments are conducted on the MNIST \cite{mnist} and CIFAR-10 datasets, utilizing ReLU activation functions, the cross-entropy loss, the Adam optimizer, and default and recommended hyperparameters. We us weight decay for regularization such that the Lipschitz constants are reasonably low. Details on the architectures are deferred to Appendix~\ref{app:NNarchitectures}. In Table \ref{table:main_results2}, we compare our method (GLipSDP) to LipLT \cite{fazlyab2024certified}, the trivial matrix product bound (MP) and two variations of LipSDP and GLipSDP:

\begin{itemize}
  \item \textbf{S-LipSDP}:  As LipSDP runs into memory issues for the chosen architectures, we apply LipSDP on possibly large subnetworks that are analyzed separately. The product of the Lipschitz estimates for the subnetworks yields an upper bound for the network. 
  \item \textbf{S-GLipSDP}: We also compute S-GLipSDP (Split GLipSDP), splitting the NN into subnetworks that are convenient to handle and apply GLipSDP to the subnetworks. Again, the product of the bounds of the subnetworks gives an upper bound for the entire NN. 
\end{itemize}
All splits into subnetworks for S-GLipSDP and S-LipSDP are listed in Table~\ref{table:splits_S-GLipSDP} and \ref{table:splits_S-LipSDP}, respectively, in Appendix~\ref{app:NNarchitectures}. To illustrate the remaining conservatism of the bounds, we also empirically compute lower bounds.

In Table \ref{table:main_results2}, we summarize the resulting Lipschitz bounds and computation times for the different NNs. We observe that GLipSDP outperforms the S-GLipSDP, S-LipSDP, LipLT, and MP bounds in all NNs. If a value is not stated, the corresponding method runs into memory issues or is not applicable. Especially for the FC-R18, a residual NN with fully connected layers, GLipSDP achieves a bound very close to the empirical lower bound and in other NNs that involve convolutions, there also are significant improvements over MP bounds and clear advantages compared to LipLT. LipLT and LipSDP do not accommodate pooling layers, which are a building block of LeNet-5, showing the versatility of GLipSDP in comparison to these other methods. In the LeNet-5 examples, while GLipSDP provides the best bounds, we notice that S-LipSDP gives better bounds than S-GLipSDP. This is due to the conservatism introduced through (ii) and (iii) in handling convolutional layers, cmp. Remark~\ref{rem:recasting}. On the example of 2C2F, we recognize a clear computational advantage of GLipSDP and S-GLipSDP over S-LipSDP together with the tighter bound found with GLipSDP. Depending on the chosen split of the subnetworks the computation time of S-GLipSDP can also be larger than the one for S-LipSDP, as for example for 4C3F (the splits are different for S-GLipSDP and S-LipSDP, see Tables \ref{table:splits_S-GLipSDP} and \ref{table:splits_S-LipSDP}). In this example, we, however, achieve a much lower bound with S-GLipSDP. Using a 32GB RAM note book, GLipSDP runs into memory issues on 4C3F and 6C2F, yet using S-GLipSDP, we still achieve much better bounds than the MP bounds and also outperform LipLT in our 4C3F example. In summary, our newly proposed method enables us to achieve tighter bounds than those provided by LipLT and MP, and it is more scalable than other SDP-based methods.

\begin{table*}[t]
  \centering
  \caption{Lipschitz bounds (computation times in seconds) for different models with stated accuracies.
}
    %\resizebox{0.7\textwidth}{!}{%
    \begin{tabular}{llrrrrrrrrrr}
    \toprule
    \textbf{Dataset} & \textbf{Model}  & \textbf{Acc.} & \textbf{Emp. LB} & \multicolumn{2}{c}{\textbf{GLipSDP}} & \multicolumn{2}{c}{\textbf{S-GLipSDP}} & \multicolumn{2}{c}{\textbf{S-LipSDP}} & \textbf{LipLT} & \textbf{MP} \\
    \midrule 
    \multirow{5}{*}{MNIST} 
     %& FC & 94.2\% &\\
     & LeNet-5  & 97.5\% & 3.607 & {\bf 12.252} & (99) & 15.855 & (85) & 15.086 & (10) & -- & 15.620 \\
     & 2C2F  & 97.2\% & 3.388 &  {\bf 6.725} &  (360) & 8.223 & (232) & 7.689 & (7526) & 7.136 & 10.892\\
     & 4C3F  & 96.2\% & 5.077 & -- &  &  {\bf 22.240} & (14573) & 53.739 & (293) & 26.021 & 55.237 \\
     & FC-R18 & 96.0\% & 13.32 & {\bf 13.644} & (229) & 24.372 & (4) & 24.893 & (1) & 18.619 & 25.973\\
     & C-R18 & 96.5\% & 15.315 &{\bf 101.426} & (89) & 176.329 & (37) & -- & & 132.048 & 343.149 \\
     \midrule
     \multirow{2}{*}{CIFAR-10} 
     & LeNet-5  & 57.1\% & 9.9 & {\bf 82.209} & (108) & 107.728 & (114) &  99.839 & (19) & -- & 109.543 \\
     & 6C2F     & 64.2\% &  31.091   & -- &  & 1273.2 & (26857) & -- & & {\bf 786.919} &  2057.4  \\
    \bottomrule 
    \end{tabular}% 
    %} 
  \label{table:main_results2}%
\end{table*}%

\section{Conclusion}
We presented a versatile and scalable approach for Lipschitz constant estimation for a large class of NN architectures. Our approach views the NN as a time-varying dynamical system, where we interpret the layer indices as time indices. This view allows us to exploit the layer-wise composition structure of NNs. In addition, we leverage the structure of the individual layers, especially of convolutional layers that we represent as N-D systems of the Roesser type. Our method is more accurate than other optimization-free methods and more scalable than previous SDP-based methods.

\section*{Acknowledgments}
The authors thank Andrea Iannelli for helpful comments on the manuscript.

{\appendices

\section{Additional proofs}
%In this section, we provide proofs for Lemma \ref{lem:diss_CNN} and Theorem~\ref{app:proof_proposition}.
\subsection{Proof of Lemma \ref{lem:diss_CNN}}
\label{app:proof_lem_diss_CNN}

We assume that the convolutional layer $\calL = \calC$ is realized as a Roesser system \eqref{eq:RoesserSys}. This means that for any $(u\lsb \bi \rsb ) \in \signals{d}{c_{-}}$, there exists a uniquely defined $(x\lsb \bi \rsb ) \in \signals{d}{n}$ with $x\lsb \bi \rsb ^\top = \begin{bmatrix}
    x_1 \lsb \bi \rsb ^\top & \cdots & x_{d} \lsb \bi \rsb ^\top
\end{bmatrix}$ and
\begin{align}\label{eq:zero_initial_conditions}
    x_j \lsb \bi \rsb &= 0 & \forall \bi \in \bbN_0^{d}, i_j = 0,
\end{align}
such that $(u\lsb \bi \rsb , x\lsb \bi \rsb , y \lsb \bi \rsb )$ with $y = \calC (u)$ satisfies \eqref{eq:RoesserSys}, where $i_j$ denotes the $j$-th index in $\bi$. 

Hence, let two arbitrary inputs $(u^1\lsb \bi \rsb ), (u^2\lsb \bi \rsb ) \in \signals{d}{c_{-}}$ be given and let $(y^1\lsb \bi \rsb )$, $(y^2\lsb \bi \rsb )$, $(x^1\lsb \bi \rsb )$, $(x^2\lsb \bi \rsb )$ denote the corresponding state and output response of the layer $\calC$. %We use the subscript in $x_j[\bi]$ to count through $j=1,\dots,d$.  
Multiplying the matrix inequality \eqref{eq:cert_CNN} from the left by the vector
$\begin{bmatrix}
    \Delta x \lsb \bi \rsb^\top & \Delta u \lsb \bi \rsb^\top
\end{bmatrix}$,
where $\Delta x \lsb \bi \rsb = x^1\lsb \bi \rsb-x^2\lsb \bi \rsb$ and $\Delta u \lsb \bi \rsb = u^1\lsb \bi \rsb - u^2\lsb \bi \rsb$, and from the right by its transpose yields the inequality
\begin{align*}
    &\sum_{j=1}^d \Delta x_j\lsb \bi \rsb ^\top P_j \Delta x_j\lsb \bi \rsb +\Delta u \lsb \bi \rsb ^\top \widetilde{X}_{-} \Delta u \lsb \bi \rsb \\
    &\geq
    \sum_{j=1}^d \Delta x_j \lsb \bi + e_j \rsb ^\top P_j \Delta x_j \lsb \bi + e_j  \rsb +
    \Delta y \lsb \bi \rsb ^\top \widetilde{X} \Delta y \lsb \bi \rsb.
\end{align*}
Note that the bias terms do not need to be considered, since they cancel out when computing the differences $\Delta x_j\lsb \bi + e_j \rsb$.
Summing this inequality over all $\bi \in \bbN_0^{d}$ yields
\begin{align*}
    &\sum_{\bi \in \bbN_0^d}\sum_{j=1}^d \Delta x_j\lsb \bi \rsb^\top P_j \Delta x_j\lsb \bi \rsb
    +
    \sum_{\bi \in \bbN_0^d} \Delta u\lsb \bi \rsb^\top \widetilde{X}_{-} \Delta u\lsb \bi \rsb\\
    &\geq
    \sum_{j=1}^d\sum_{\bi \in \bbN_0^d,i_j \geq 1} \Delta x_j\lsb \bi \rsb^\top P_j \Delta x_j\lsb \bi \rsb +
    \sum_{\bi \in \bbN_0^d} \Delta y \lsb \bi \rsb^\top \widetilde{X} \Delta y \lsb \bi \rsb,
\end{align*}
wherein $\Delta y\lsb \bi \rsb = y^1\lsb \bi \rsb - y^2\lsb \bi \rsb$. These sums all converge since all signals are in $\signals{d}{}$,  %since all signals are in $\signals{d_{k-1}}{}$
as the convolutional layer is a finite impulse response filter and, therefore, it is stable. Canceling terms on both sides yields
\begin{align*}
    &\sum_{\bi \in \bbN_0^d} \Delta u \lsb \bi \rsb^\top \widetilde{X}_{-} \Delta u\lsb \bi \rsb \\
    &\geq
    \sum_{j=1}^d\sum_{\bi \in \bbN_0^d,i_j = 0} \Delta x_j \lsb \bi\rsb^\top P_j \Delta x_j \lsb \bi \rsb +
    \sum_{\bi \in \bbN_0^d} \Delta y \lsb \bi \rsb^\top \widetilde{X} \Delta y \lsb \bi \rsb.
\end{align*}
Here, the sum over the boundary terms $\sum_{j=1}^d\sum_{\bi \in \bbN_0^d,i_j = 0} \Delta x_j\lsb \bi \rsb^\top P_j \Delta x_j \lsb \bi \rsb$ is zero, cmp. \eqref{eq:zero_initial_conditions}, such that this inequality is exactly what we had to show.

\subsection{Proof of Lemma~\ref{lem:GroupSort}} \label{app:GroupSort}
%\begin{proof}
    %For ease of notation, we assume $\calD_{k-1}=\bbR^{c}$ throughout the proof, i.e., $\widetilde{X}_{k-1}=X_{k-1}$, $\widetilde{X}_k=X_k$. The extension to $\calD_{k-1}=\ell_{2e}^{c}(\bbN_0^d)$ is straightforward given that the activation function is applied to all vectors $u[\bi]\in\bbR^c$, $\bi\in\bbN_0^d$. 
    To prove Lemma~\ref{lem:GroupSort}, we require the following lemma, which is a simplification of \cite[Lemma 1]{pauli2024novel} that directly follows for $P=S=0$ in \cite[Lemma 1]{pauli2024novel}.
    \begin{lemma} \label{lemma:GroupSort2}
        Consider a GroupSort activation $\sigma^{\mathrm{GS}}:\bbR^c\to\bbR^c$ with group size $n_g$. %Let $N=\frac{c}{n_g}$.
        For any  $T\in\calT_{n_g}^c$, $\sigma^{\mathrm{GS}}$ satisfies
    \begin{equation*}
        \begin{bmatrix}
            x-y\\
            \sigma(x)-\sigma(y)
        \end{bmatrix}^\top
        \begin{bmatrix}
            T & 0\\
            0 & -T
        \end{bmatrix}
        \begin{bmatrix}
            x-y\\
            \sigma(x)-\sigma(y)
        \end{bmatrix}\geq 0, ~ \forall~ x,y\in\bbR^c.
    \end{equation*}
    \end{lemma}
    If $\widetilde{X}_{-}\in\calT_{n_g}^c$ and $\widetilde{X}\in\calT_{n_g}^c$ satisfy $0\preceq \widetilde{X}\preceq \widetilde{X}_{-}$, there exists a multiplier $T\in\calT_{n_g}^c$ that satisfies
    $0\preceq \widetilde{X}\preceq T \preceq \widetilde{X}_{-}$, for which we equivalently write
    \begin{equation}\label{eq:GroupSort}
        \begin{bmatrix}
            \widetilde{X}_{-} - T & 0\\
            0 & -\widetilde{X} + T
        \end{bmatrix}\succeq 0.
    \end{equation}
    Let $(u^1\lsb \bi \rsb ), (u^2\lsb \bi \rsb ) \in \signals{d}{c_{-}}$ be two arbitrary inputs with corresponding outputs $(y^1\lsb \bi \rsb ) , (y^2\lsb \bi \rsb )$ of the GroupSort activation layer. We multiply \eqref{eq:GroupSort} with $\begin{bmatrix} \Delta u \lsb \bi \rsb^\top &   \Delta y \lsb \bi \rsb^\top \end{bmatrix}$ from the left, where $\Delta x \lsb \bi \rsb = x^1\lsb \bi \rsb-x^2\lsb \bi \rsb$ and $\Delta u \lsb \bi \rsb = u^1\lsb \bi \rsb - u^2\lsb \bi \rsb$, and its transpose from the right and further sum over $\bi\in\bbN_0^d$, to obtain
    \begin{align*}
        &V_{X_{-}}(u^1,u^2) -
        V_{X}(y^1,y^2)\\  &\geq 
        \sum_{\bi\in\bbN_0^d}
        \begin{bmatrix}
            \Delta u \lsb \bi \rsb\\
            \Delta y \lsb \bi \rsb
        \end{bmatrix}^\top
        \begin{bmatrix}
            T & 0\\
            0 & -T
        \end{bmatrix}
        \begin{bmatrix}
            \Delta u \lsb \bi \rsb\\
            \Delta y \lsb \bi \rsb
        \end{bmatrix} \geq 0,
    \end{align*}
    where the last inequality follows from Lemma \ref{lemma:GroupSort2}.
%\end{proof}

\subsection{Proof of Lemma~\ref{lem:conv+act}}\label{sec:conv+act}
%\begin{proof}
Let two arbitrary inputs $(u^1\lsb \bi \rsb ) , (u^2\lsb \bi \rsb ) \in \signals{d}{c_{-}}$ be given and let $(y^1\lsb \bi \rsb ) , (y^2\lsb \bi \rsb )$, $(x^1\lsb \bi \rsb ) , (x^2\lsb \bi \rsb )$ denote the corresponding state and output response of the layer $\calC$, where $x^m\lsb \bi \rsb = \begin{bmatrix}
    x^m_1 \lsb \bi \rsb ^\top & \cdots & x^m_{d} \lsb \bi \rsb ^\top
\end{bmatrix}^\top$, $m=1,2$. We left/right multiply \eqref{eq:cert_CNN} with
$\begin{bmatrix}
    \Delta x \lsb \bi \rsb ^\top & \Delta u \lsb \bi \rsb ^\top & \Delta y\lsb \bi \rsb^\top
\end{bmatrix}$,
where $\Delta x \lsb \bi \rsb = x^1\lsb \bi \rsb-x^2\lsb \bi \rsb$, $\Delta u \lsb \bi \rsb = u^1\lsb \bi \rsb - u^2\lsb \bi \rsb$ and $\Delta y \lsb \bi \rsb = y^1\lsb \bi \rsb - y^2\lsb \bi \rsb$,
and its transpose, respectively, and obtain
\begin{align*}
    &\sum_{j=1}^d \Delta x_j\lsb \bi \rsb^\top P_j \Delta x_j\lsb \bi \rsb+
    \Delta u \lsb \bi \rsb^\top \widetilde{X}_{-} \Delta u \lsb \bi \rsb\\
    & + 2\Delta y \lsb \bi \rsb^\top \Lambda \Delta y \lsb \bi \rsb - 2 \Delta y \lsb \bi \rsb^\top\Lambda (\bC\Delta x \lsb \bi \rsb +\bD \Delta u \lsb \bi \rsb)\\
    &\geq
    \sum_{j=1}^d \Delta x_j\lsb \bi + e_j \rsb ^\top P_j \Delta x_j \lsb \bi + e_j  \rsb+
    \Delta y \lsb \bi \rsb ^\top \widetilde{X} \Delta y \lsb \bi \rsb.
\end{align*}
Subsequent summation over all $\bi\in\bbN_0^{d}$ then yields
\begin{align*}
    &\sum_{\bi \in \bbN_0^d} \Delta u \lsb \bi \rsb^\top \widetilde{X}_{-} \Delta u \lsb \bi \rsb\\
    & + 2\Delta y \lsb \bi \rsb^\top\Lambda \Delta y \lsb \bi \rsb- 2 \Delta y \lsb \bi \rsb^\top \Lambda (\bC\Delta x\lsb \bi \rsb+\bD\Delta u \lsb \bi \rsb)\\
    %\sum_{j=1}^d\sum_{\bi \in \bbN_0^d,i_j = 0} (x_j^1\lsb \bi\rsb - x_j^2\lsb \bi \rsb )^\top P_j (x_j^1\lsb \bi \rsb - x_j^2\lsb \bi\rsb )\\
    &\geq
    \sum_{\bi \in \bbN_0^d} \Delta y[\bi]^\top \widetilde{X} \Delta y \lsb\bi\rsb,\\
\end{align*}
again using the arguments laid out in the proof of Lemma~\ref{lem:diss_CNN}. By Lemma \ref{lem:slope_restriction}, we conclude that $2 \Delta y \lsb\bi\rsb^\top\Lambda \Delta y \lsb\bi\rsb- 2 \Delta y \lsb\bi\rsb^\top\Lambda (\bC\Delta x\lsb\bi\rsb +\bD \Delta u\lsb \bi \rsb)\leq 0$ for all $\bi\in\bbN_0^{d}$ such that we obtain \eqref{eq:metricOpt_DP_relaxed}.

\subsection{Proof of Lemma~\ref{lem:ResNet_1_layer}}\label{app:ResNet_cert}

    For some $u^1,u^2\in\bbR^{c}$ and the corresponding intermediate outputs $v^1,v^2\in\bbR^{c}$, and outputs $y^1,y^2\in\bbR^{c}$ of the ResNet layer \eqref{eq:ResNet_layer}, we left/right multiply \eqref{eq:cert_ResNet} with 
    $\begin{bmatrix}
        \Delta u^\top &  \Delta v^\top
    \end{bmatrix}$,
    where $\Delta u=u^1-u^2$ and $\Delta v=v^1-v^2$, and its transpose, respectively. %, where $v_k^i=\sigma(W_1 u_k^i+b_k)$, $i=1,2$.
    We obtain
    \begin{align*}
        V_{X_{-}}(u^1,u^2)-V_{X}(y^1,y^2) \geq 
        -2 \Delta v^\top \Lambda\Delta v + 2 \Delta v^\top \Lambda W_1 \Delta u.
    \end{align*}
    Given that the activation functions are slope-restricted on $[0,1]$, we use Lemma~\ref{lem:slope_restriction} to conclude that $-2 \Delta v^\top \Lambda\Delta v + 2 \Delta v^\top \Lambda W_1 \Delta u\geq 0$. It follows that $V_{X_{-}}(u^1,u^2)-V_{X}(y^1,y^2)\geq 0$.

\subsection{Proof of Theorem \ref{thm:noMoreConservativeThanQSR}}\label{app:proof_proposition}
We prove Theorem \ref{thm:noMoreConservativeThanQSR} by induction.\\
\textbf{Induction hypothesis:} If for some $(Q_1,\dots,Q_l)$, $(R_1,\dots,R_l)$, $(S_1,\dots,S_l)$, $Z_l$ and $\gamma > 0$ 
\begin{align}\label{eq:proofThm1_hypothesis}
{\small
	\begin{bmatrix}
		Q_{l} - Z_l & S_{l}\\
		S_{l}^\top & R_{l} + Q_{l-1} & S_{l-1}\\
		& S_{l-1}^\top & R_{l-1} + Q_{l-2} & \ddots\\
		  & & \ddots & \ddots & S_1\\
		  & & & S_1^\top & R_1 + \gamma^2 I
	\end{bmatrix} \preceq 0}
\end{align}
is satisfied, then there exists a sequence of matrices $(X_0,\dots,X_l)$ such that $X_0=\gamma^2 I$, $X_l=Z_l$ and \eqref{eq:LMI_Thm1} holds.\\
\textbf{Start of induction:} $l = 1$. Assume
\begin{equation*}
    \begin{bmatrix}
        Q_1-Z_1 & S_1\\
        S_1^\top & R_1+\gamma^2 I    
    \end{bmatrix}\preceq 0
\end{equation*}
holds for some $Q_1$, $R_1$, $S_1$, $\gamma>0$, and $Z_1 \succeq 0$. Then \eqref{eq:LMI_Thm1} is satisfied with $X_0=\gamma^2 I$, $X_1=Z_1$:
\begin{equation*}
    \begin{bmatrix}
        Q_1 & S_1\\
        S_1^\top & R_1    
    \end{bmatrix}\preceq
    \begin{bmatrix}
        Z_1 & 0\\
        0 & -\gamma^2 I    
    \end{bmatrix}.
\end{equation*}\\
\textbf{Induction step:} $l\to l+1$. Assume that our induction hypothesis holds for $l$. Let  for some $(Q_1,\dots,Q_{l+1})$, $(R_1,\dots,R_{l+1})$, $(S_1,\dots,S_{l+1})$, $Z_{l+1}$, $\gamma > 0$ the inequality
\begin{align}\label{eq:proofThm1}
{\small
	\begin{bmatrix}
		Q_{l+1} - Z_{l+1} & S_{l+1}\\
		S_{l+1}^\top & R_{l+1} + Q_{l} & S_{l}\\
		& S_{l}^\top & R_{l} + Q_{l-1} & \ddots\\
		  & & \ddots & \ddots & S_1\\
		  & & & S_1^\top & R_1 + \gamma^2 I
	\end{bmatrix}\!\preceq\!0}
\end{align}
hold, which implies $Q_{l+1} - Z_{l+1}\preceq0$. There exists an orthogonal matrix $V$ of the eigenvectors of $Q_{l+1} - Z_{l+1}$ that diagonalizes $Q_{l+1} - Z_{l+1}$ by a similarity transformation, i.\,e., $V^\top(Q_{l+1} - Z_{l+1})V$ is a diagonal matrix. We construct $V=\begin{bmatrix} V_1 & V_2\end{bmatrix}$ in such a way that $V^\top(Q_{l+1} - Z_{l+1})V=\diag(0,\dots,0,v_1,\dots,v_n)=\blkdiag(0,D)$, $V_1^\top (Q_{l+1} - Z_{l+1}) V_1 = 0$, $V_2^\top (Q_{l+1} - Z_{l+1}) V_2 = D$, where $v_1,\dots,v_n<0$ and $n$ is the rank of $Q_{l+1} - Z_{l+1}$. Next, we left and right multiply \eqref{eq:proofThm1} with the full-rank matrix $\diag(V^\top,I)$ and its transpose $\diag(V,I)$, respectively, which yields
\begin{align}\label{eq:proofThm1_2}
{\small
	\begin{bmatrix}
		\begin{bmatrix} 0 & 0 \\ 0 & D \end{bmatrix} & \begin{bmatrix}  0 \\ V_2^\top S_{l+1} \end{bmatrix}\\
		\begin{bmatrix} 0 & S_{l+1}^\top V_2 \end{bmatrix}  & R_{l+1} + Q_{l} & \ddots\\
		  & \ddots & \ddots & S_1\\
		  & & S_1^\top & R_1 + \gamma^2 I
	\end{bmatrix} \preceq 0}
\end{align}
and further, we drop the $c_{l+1}-n$ zero rows and columns of \eqref{eq:proofThm1_2}, resulting in
\begin{align}\label{eq:proofThm1_3}
{\small
	\begin{bmatrix}
		D & V_2^\top S_{l+1}\\
		S_{l+1}^\top V_2 & R_{l+1} + Q_{l} & S_{l}\\
		& S_{l}^\top & R_{l} + Q_{l-1} & \ddots\\
		  & & \ddots & \ddots & S_1\\
		  & & & S_1^\top & R_1 + \gamma^2 I
	\end{bmatrix}\!\preceq\!0.}
\end{align}
We now apply the Schur complement to \eqref{eq:proofThm1_3} w.r.t.~$D$, which yields that \eqref{eq:proofThm1_3} is negative semi-definite if and only if \eqref{eq:proofThm1_hypothesis} and $D\prec 0$ hold, where $Z_l=S_{l+1}^\top V_2 D^{-1} V_2^\top S_{l+1}-R_{l+1}$. Given that the diagonal matrix $D$ has only entries $v_1,\dots,v_n<0$, $D\prec 0$ is satisfied and $D$ is invertible. By the induction hypothesis, there exists a sequence of matrices $X_0,\dots,X_{l}$ such that $X_0=\gamma^2 I$, $X_l=Z_l$, \eqref{eq:LMI_Thm1}. The equality $X_l=Z_l$ implies that there exists at least one $X_l$ that satisfies $X_l\preceq Z_l$ and $\begin{bsmallmatrix}
        Q_l & S_l\\
        S_l^\top & R_l
    \end{bsmallmatrix} \preceq \begin{bsmallmatrix}
        X_l & 0\\
        0 & -X_{l-1}
    \end{bsmallmatrix}$.
    Here, $X_l\preceq Z_l$ reads $X_l\preceq S_{l+1}^\top V_2 D^{-1} V_2^\top S_{l+1}-R_{l+1}$. By the Schur complement, we then get
\begin{equation*}
    \begin{bmatrix}
        D & V_2^\top S_{l+1}\\
        S_{l+1}^\top V_2 & R_{l+1}+X_l
    \end{bmatrix}\preceq 0,
\end{equation*}
to which we again add the dropped $c_{l+1}-n$ zero rows and columns, yielding
\begin{equation}\label{eq:proofThm1_5}
    \begin{bmatrix}
        V^\top (Q_{l} - Z_l)V & V^\top S_{l+1}\\
        S_{l+1}^\top V & R_{l+1}+X_l
    \end{bmatrix}\preceq 0.
\end{equation}
Subsequently, we left and right multiply \eqref{eq:proofThm1_5} with $\diag(V, I)$ and its transpose $\diag(V^\top,I)$, respectively, yielding
\begin{equation*}
    \begin{bmatrix}
        Q_{l+1}-Z_{l+1} & S_{l+1}\\
        S_{l+1}^\top & R_{l+1}+X_l
    \end{bmatrix}\preceq 0,
\end{equation*}
and we further set $X_{l+1}=Z_{l+1}$, which concludes the induction step.

The statement of the theorem is a special case of our induction hypothesis for $Z_l = I$.

\section{Strided convolutions}\label{sec:strided_convolutions}
To represent strided convolutions in state space, we require a reshaping operator as a strided convolution is only shift invariant w.r.t. a shift by the stride $\bs$ along $\bi$. The strided convolution then takes the form $\calL = \calC \circ \calR_{\bs}$ with a convolutional layer $\calC$ with stride one and a reshaping operator $\calR_{\bs}$. This reshaping operator $\calR_{\bs}$ is given by
\begin{align*}
    \signals{d_{-}}{c_{-}} \to& \signals{d_{-}}{c_{-}|[1,\bs]|},\\
    &(u\lsb \bi \rsb) \mapsto 
    (\mathrm{vec} (u\lsb \bs\bi + \bt \rsb \mid \bt \in [0,\bs[)),
\end{align*}
where $\mathrm{vec} (u\lsb \bs\bi + \bt \rsb \mid \bt \in [0,\bs[)$ denotes the stacked vector of the signal entries $u\lsb \bs\bi + \bt \rsb , \bt \in [0,\bs[$. The resulting state space representation for a strided convolution takes this stacked vector $\mathrm{vec} (u\lsb \bs\bi + \bt \rsb \mid \bt \in [0,\bs[)$ as its input. We can view this as the flattening of a batch of pixels into a vectorized input which in turn serves as the input to the state space representation of the strided convolution. Details on the construction of the Roesser model for strided convolutions and multiple examples can be found in \cite{pauli2024state}.

    To use Lemma \ref{lem:diss_CNN} on strided convolutions we recall that strided convolutions take the form $\calL = \calC \circ \calR_{\bs}$. Hence, for the layer $\calL$ and the value function $V_{X}$, \eqref{eq:metricOpt_DP_relaxed} takes the form
    \begin{align*}
        \langle u,X_{-}u \rangle \geq \langle (\calC\circ \calR_{\bs}) u, X (\calC\circ \calR_{\bs}) u \rangle
    \end{align*}
    for all $u \in \signals{d_{-}}{c_{-}}$ or, equivalently,
    \begin{align*}
        \langle u',\calR_{\bs}^{-\dagger}X_{-}\calR_{\bs}^{-1}u' \rangle \geq \langle \calC u', X \calC u' \rangle 
    \end{align*}
    for all $u' \in \signals{d_{-}}{c_{-}|[0,\bs[|}$, where $\calR_{\bs}^{-1}$ is the inverse and $\calR_{\bs}^{-\dagger}$ the inverse adjoint of $\calR_{\bs}$. The last inequality is of the same form as \eqref{eq:metricOpt_DP_relaxed} with $\calR_{\bs}^{-\dagger}X_{-}\calR_{\bs}^{-1}$ replacing $X_{-}$. Consequently, we need to enforce $\calR_{\bs}^{-\dagger}X_{-}\calR_{\bs}^{-1} \in \calH_\calC^u$ to be able to apply Lemma \ref{lem:diss_CNN}.

    Basic reshaping arguments yield that this condition is equivalent to
    \begin{align}
    \resizebox{\linewidth}{!}{$
        \begin{bmatrix}
            X_{-}'[\bs\bi,\bs\bj] &  \cdots & X_{-}'[\bs\bi + \bs - 1,\bs\bj]\\
            \vdots & \ddots & \vdots\\
            X_{-}'[\bs\bi,\bs\bj+\bs - 1] &  \cdots & X_{-}'[\bs\bi + \bs - 1,\bs\bj+\bs - 1]
        \end{bmatrix}$}
        \label{eq:stridedRestriction}
    \end{align}
    being equal to some symmetric matrix $Y\in \bbR^{c_{-}|[0,\bs[| \times c_{-}|[0,\bs[|}$ for $\bi = \bj$ or zero for $\bi \neq \bj$. For matrices $X_{-} \in \calH_\calC^y$, we have $X_{-}'[\bi,\bj] = \widetilde{X}_{-}$ for $\bi = \bj$ and $X_{-}'[\bi,\bj]=0$ for $\bi \neq \bj$. This implies that \eqref{eq:stridedRestriction} is satisfied with $Y = \diag(\widetilde{X}_{-},\ldots,\widetilde{X}_{-})$.

\section{LMI constraints for subnetworks}\label{app:subnetworks}
Usually we consider the combination of a linear layer with a nonlinear activation function as shown in Section \ref{sec:statespace} and formulate LMI constraints for this combination. However, combining multiple layers is also possible. While producing larger LMI constraints, we renounce the use of the decision variables at the transition of layers, i.e., $X$, which reduces the number of decision variables. The following LMIs state the corresponding constraints.

\begin{lemma}\label{lem:FC_multiple_layers}
    Consider a fully connected subnetwork $\sigma\circ\calF_{l}\circ\dots\circ\sigma\circ\calF_1$ with activation functions that are slope-restricted on $[0,1]$. For some $X \in \calH_{\calF}^y$ and $X_{-} \in \calH_{\calF}^u$, this subnetwork satisfies \eqref{eq:metricOpt_DP_relaxed} if there exist $\Lambda_j\in\bbD_{++}^{n_{y_j}},~j=1,\dots,l$, such that %$\calG_{\calF}(X_{-},X,\nu):=$
    \begin{equation}\label{eq:fc}
        \begin{bmatrix}
            X_{-}       & -W_1^\top\Lambda_1 & 0                   & \cdots         & 0\\
            -\Lambda_1W_1 & 2\Lambda_1         & -W_2^\top\Lambda_2  & \ddots         & \vdots\\
            0             & -\Lambda_2W_2      & \ddots              & \ddots         & 0 \\
            \vdots        & \ddots             & \ddots              & 2\Lambda_{l-1} & -W_l^\top\Lambda_l\\
            \phantom{\ddots}0\phantom{\ddots}             & \cdots             & 0                   & -\Lambda_lW_l  & 2\Lambda_l-X
        \end{bmatrix}\succeq 0.
    \end{equation}
\end{lemma}

\begin{lemma}\label{lem:Conv_multiple_layers}
    Consider a fully convolutional subnetwork $\sigma_l\circ\calC_{l-1}\circ\dots\circ\sigma_2\circ\calC_1$ with activation functions that are slope-restricted on $[0,1]$. For some $X \in \calH_{\calC}^y$ and $X_{-} \in \calH_{\calC}^u$, this subnetwork satisfies \eqref{eq:metricOpt_DP_relaxed} if there exist $\Lambda_j\in\bbD_{++}^{c_j}$, $\bP_j=\blkdiag(P_1^j,\dots,P_d^j),~P_i^j\in\bbS_{++}^{n_i},~i=1,\dots,d,~j=1,\dots,l$ such that \eqref{eq:conv}.
\end{lemma}

Proofs of Lemmas \ref{lem:FC_multiple_layers} and \ref{lem:Conv_multiple_layers} follow the same arguments as previous proofs and coincide with the proofs in \cite{fazlyab2019efficient,gramlich2023convolutional} for the case that the whole NN is considered as a subnetwork.

\begin{figure*}
\scriptsize
\begin{equation}\label{eq:conv}
    \begin{split}
    %&\calG_{\calC}(X_{-},X,\nu):=\\
    &\begin{bmatrix}
        X_{-}-\bB_1^\top \bP_1 \bB_1 & -\bB_1^\top \bP_1 \bA_1 & -\bD_1^\top\Lambda_1    &\phantom{\ddots}             \\
        -\bA_1^\top \bP_1\bB_1     & \bP_1-\bA_1^\top \bP_1\bA_1  & -\bC_1^\top\Lambda_1            &\phantom{\ddots} \\
        -\Lambda_1 \bD_1     & -\Lambda_1 \bC_1     & 2\Lambda_1-\bB_2^\top \bP_2 \bB_2 & -\bB_2^\top \bP_2\bA_2         & -\bD_2^\top\Lambda_2 &\phantom{\ddots} \\
                             &                      & -\bA_2^\top \bP_2 \bB_2           & \bP_2-\bA_2^\top \bP_2 \bA_2   & -\bC_2^\top\Lambda_2 &\phantom{\ddots} \\
                             &                      & -\Lambda_2 \bD_2            &  -\Lambda_2 \bC_2           & 2\Lambda_2-\bB_3^\top\bP_3 \bB_3 & \ddots \\
                             &                      &                             &                             & \ddots & \ddots & -\bB_l^\top \bP_l\bA_l      & -\bD_l^\top\Lambda_l\\
                             &                      &                             &                             & \phantom{\ddots}  & -\bA_l^\top \bP_l \bB_l           & \bP_l-\bA_l^\top \bP_l \bA_l  & -\bC_l^\top\Lambda_l\\
                             &                      &                             &                             & \phantom{\ddots}  & -\Lambda_l \bD_l          &  -\Lambda_l \bC_l   & 2\Lambda_l-X
    \end{bmatrix}\succeq 0
    \end{split}
\end{equation}
\end{figure*}

\section{Neural network architectures}\label{app:NNarchitectures}
We analyze the well-known LeNet-5 \cite{lecun1998gradient} and other typical CNN architectures \cite{leino2021globally} as well as 18-layer residual NNs inspired by \cite{he2016deep}. To describe the NN architectures, similar to \cite{leino2021globally}, we denote a 2-D convolutional layer by $c(C,K,S)$, where $C$ is the number of output channels, $K$ the symmetric kernel size and $S$ the symmetric stride. A dense fully connected layer is denoted by $d(N)$, where $N$ is the number of output neurons. In addition, by $p(\text{type},K,S)$ we mean pooling layers of type either average or maximum, with kernel size $K$ and stride $S$.

\begin{table}[t]
  \centering
  \caption{Neural network architectures.}
    \resizebox{0.48\textwidth}{!}{%
    \begin{tabular}{ll}
    \toprule
    \textbf{Model} & \textbf{Specification}  \\
    \midrule 
        LeNet-5: &$c(6,5,1).p(\text{av},2,2).c(16,5,1).p(\text{av},2,2).d(120).d(84).d(10)$ \\
        2C2F: & $c(16,4,2).c(32,4,2).d(100).d(10)$ \\
        4C3F: & $c(32,3,1).c(32,4,2).c(64,3,1).c(64,4,2).d(512)^2.d(10)$ \\
        FC-R18: & $d(64).\mathrm{res}(64,2)^8.d(10)$\\
        C-R18: & $c(16,7,2).p(\text{max},2,2).\mathrm{res}(16,3,1,2)^8.p(\text{av},2,2).d(10)$\\\midrule
        LeNet-5: &$c(6,5,1).p(\text{max},2,2).c(16,5,1).p(\text{max},2,2).d(120).d(84).d(10)$ \\  6C2F: & $c(32,3,1)^2.c(32,4,2).c(64,3,1)^2.c(64,4,2).d(512).d(10)$ \\
    \bottomrule 
    \end{tabular}% 
    }    
  \label{table:architectures}%
\end{table}% 
We denote residual layers with convolutional layers in the residual path by $\mathrm{res}(C,K,S,L)$ where all convolutions are of the same shape and $L$ denotes the number of layers in the residual path. In addition, we denote a residual layer containing fully connected layers in the residual path by $\mathrm{res}(N,L)$, $N$ being the number of neurons and $L$ the number of skipped fully connected layers. Using the described nomenclature, we list all utilized architectures in Table~\ref{table:architectures}.

For the methods S-LipSDP and S-GlipSDP, we require suitable subnetworks, as specified in Table \ref{table:splits_S-GLipSDP} and Table \ref{table:splits_S-LipSDP}. S-LipSDP requires a split at every pooling layer as it does not allow to include pooling layers by quadratic constraints, and for 6C3F and 4C3F splits are chosen as large as possible before running into memory issues. For S-LipSDP on C-R18 and FC-R18, we apply LipSDP to the residual paths. The sum of the Lipschitz constants of the parallel paths, i.e., $1+\gamma(\text{residual path})$, provides an upper bound on the Lipschitz constant for the residual layer.

\begin{table}[t]
  \centering
  \caption{Splits into subnetworks for S-GLipSDP}
    \resizebox{0.48\textwidth}{!}{%
    \begin{tabular}{ll}
    \toprule
    \textbf{Model} & \textbf{Specification}  \\
    \midrule 
        LeNet-5: &$c(6,5,1).p(\text{av},2,2).c(16,5,1).p(\text{av},2,2)~|~ d(120).d(84).d(10)$ \\
        2C2F: & $c(16,4,2).c(32,4,2) ~|~ d(100).d(10)$ \\
        4C3F: & $c(32,3,1).c(32,4,2).c(64,3,1).c(64,4,2) ~|~ d(512).d(512).d(10)$ \\
        FC-R18: & $d(64) ~|~\mathrm{res}(64,2)~|~\cdots~|~\mathrm{res}(64,2)~|~d(10)$\\
        C-R18: & $c(16,7,2).p(\text{max},2,2)~|~\mathrm{res}(16,3,1,2)~|~\cdots~|~\mathrm{res}(16,3,1,2).p(\text{av},2,2)~|~d(10)$\\\midrule
        LeNet-5: &$c(6,5,1).p(\text{max},2,2).c(16,5,1).p(\text{max},2,2)~|~d(120).d(84).d(10)$ \\
        6C2F: & $c(32,3,1)^2.c(32,4,2)~|~c(64,3,1)^2~|~c(64,4,2)~|~d(512).d(10)$ \\
    \bottomrule 
    \end{tabular}% 
    }    
  \label{table:splits_S-GLipSDP}%
\end{table}% 

\begin{table}[t]
  \centering
  \caption{Splits into subnetworks for S-LipSDP}
    \resizebox{0.48\textwidth}{!}{%
    \begin{tabular}{ll}
    \toprule
    \textbf{Model} & \textbf{Specification}  \\
    \midrule 
        LeNet-5: &$c(6,5,1).p(\text{av},2,2)~|~c(16,5,1).p(\text{av},2,2)~|~ d(120).d(84).d(10)$ \\
        2C2F: & $c(16,4,2).c(32,4,2) ~|~ d(100).d(10)$ \\
        4C3F: & $c(32,3,1)~|~c(32,4,2)~|~c(64,3,1)~|~c(64,4,2) ~|~ d(512) ~|~ d(512).d(10)$ \\
        FC-R18: & $d(64) ~|~1+d(64).d(64)~|~\cdots~|~1+d(64).d(64)~|~d(10)$\\
        C-R18: & -- \\\midrule
        LeNet-5: &$c(6,5,1).p(\text{max},2,2) ~|~ c(16,5,1).p(\text{max},2,2)~|~d(120).d(84).d(10)$ \\
        6C2F: & -- \\
    \bottomrule 
    \end{tabular}% 
    }    
  \label{table:splits_S-LipSDP}%
\end{table}% 

\section*{References}
\vspace*{-2em}

\bibliographystyle{IEEEtran}
\bibliography{references}

%\newpage

\vspace{-33pt}
\begin{IEEEbiography}
[{\includegraphics[width=1in,height=1.25in,clip,keepaspectratio]{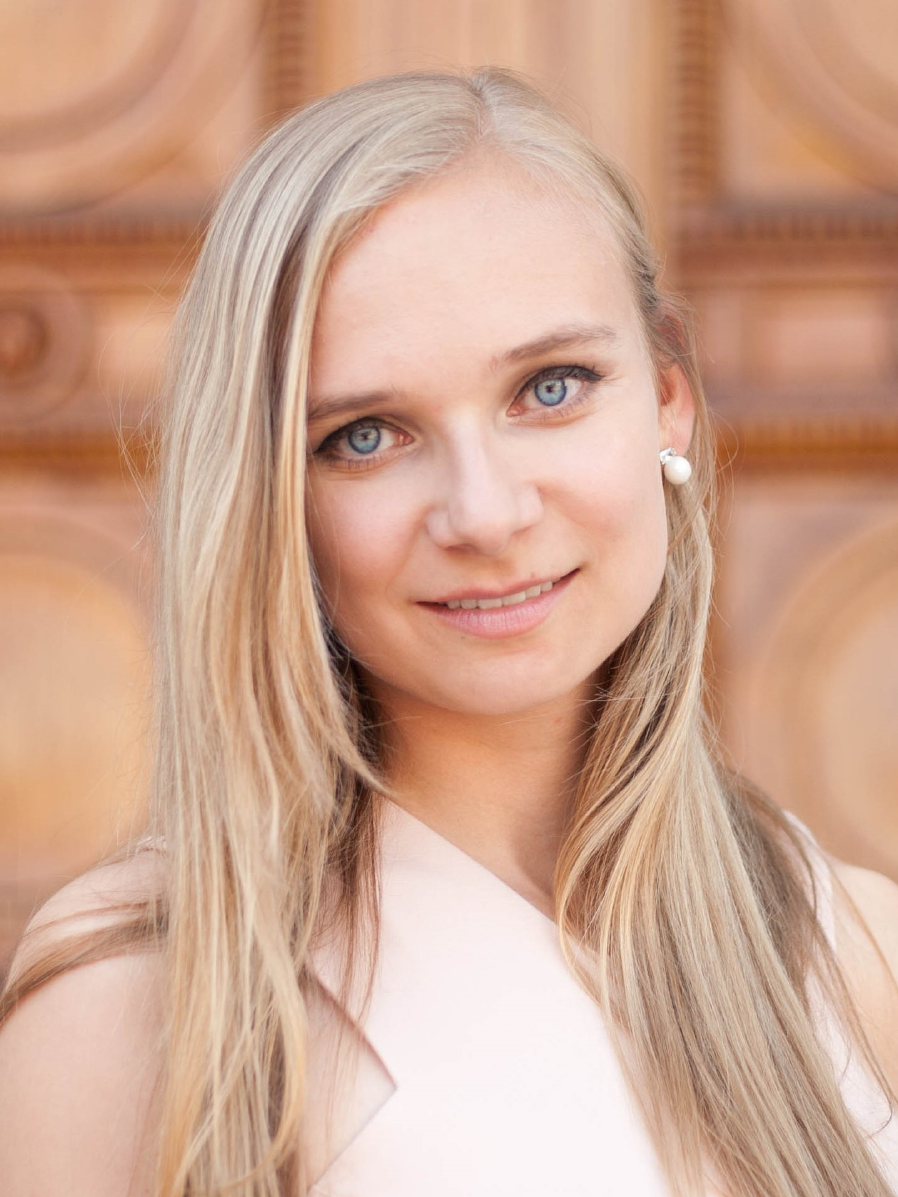}}]{Patricia Pauli} received the Master’s degree in Mechanical Engineering and Computational Engineering from the Technical University of Darmstadt, Germany, in 2019. She has since been a Ph.D. student with the Institute for Systems Theory and Automatic Control under supervision of Prof. Frank Allgöwer and a member of the International Max-Planck Research School for Intelligent Systems (IMPRS-IS). Her research interests are in the area of robust machine learning and learning-based control.
\end{IEEEbiography}

\vspace{-33pt}
\begin{IEEEbiography}
[{\includegraphics[width=1in,height=1.25in,clip,keepaspectratio]{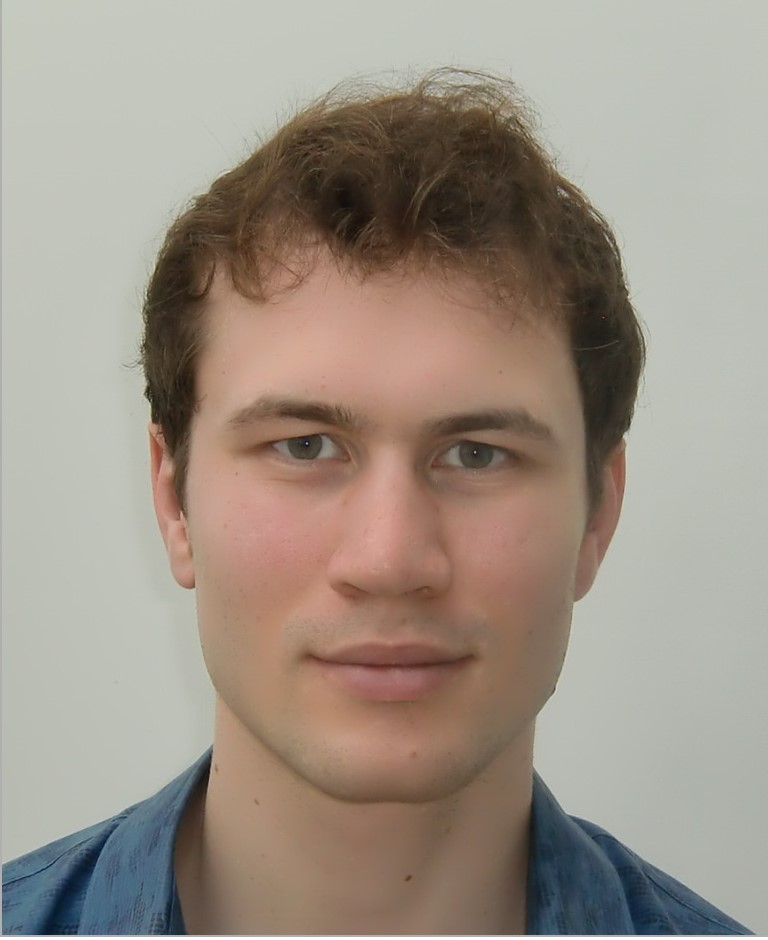}}]{Dennis Gramlich} received the Master's degree in Engineering Cybernatics and Mathematics from the University of Stuttgart, Germany, in 2020. He was a Ph.D. student with the Institute for Systems Theory and Automatic Control at the University of Stuttgart under the supervision of Prof. Christian Ebenbauer from May 2020 to October 2021 and is now a Ph.D. student with the Institute for Intelligent Control at RWTH Aachen University under the supervision of Prof. Christian Ebenbauer. His research interests are Robust Control and Robust Trajectory Optimization.
\end{IEEEbiography}

\vspace{-33pt}
\begin{IEEEbiography}
[{\includegraphics[width=1in,height=1.25in,clip,keepaspectratio]{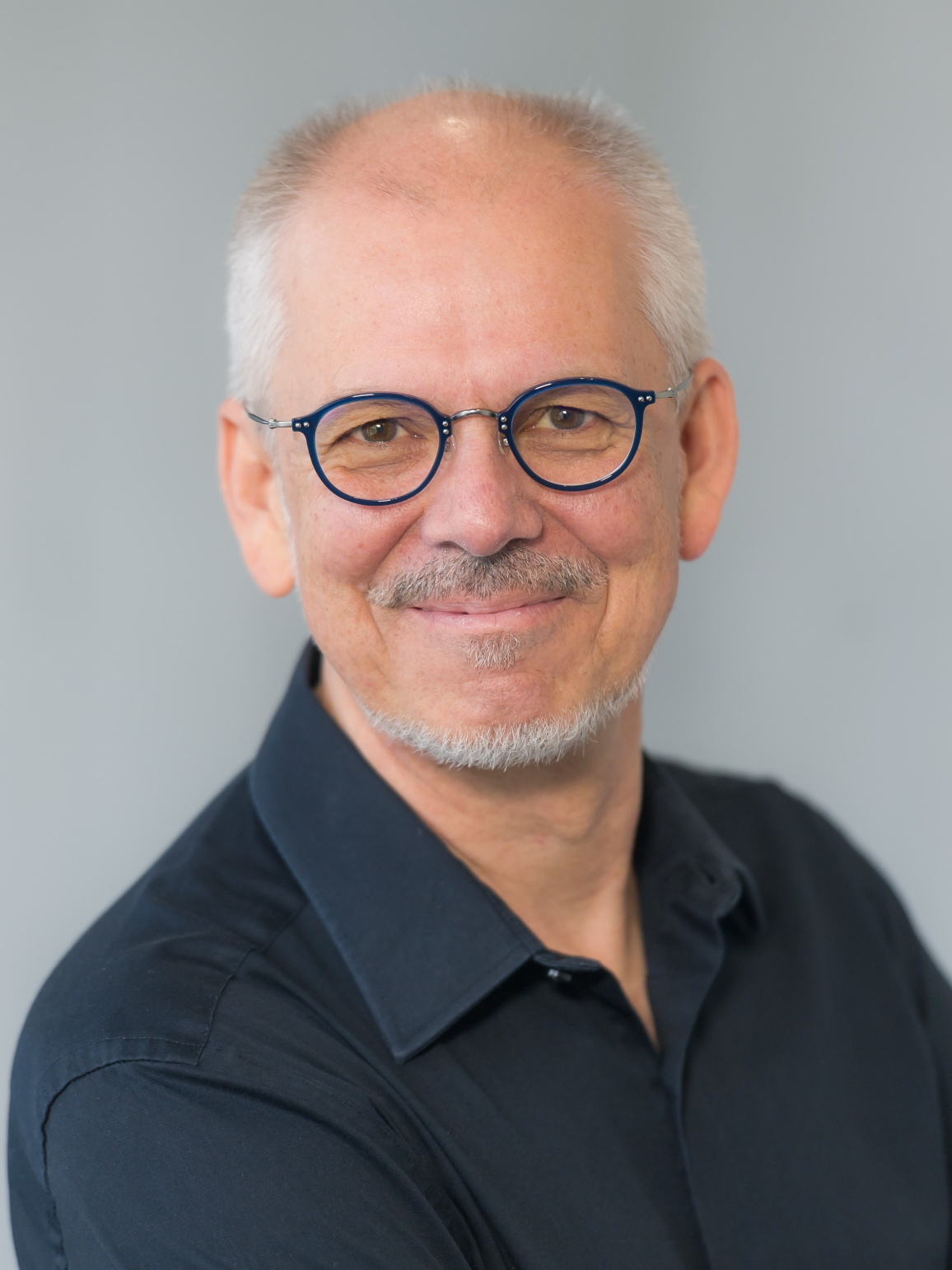}}]{Frank Allgöwer} studied Engineering Cybernetics and Applied 
Mathematics in Stuttgart and at the University of California, 
Los Angeles (UCLA), respectively, and received his Ph.D. degree 
from the University of Stuttgart in Germany. Since 1999 he is
the Director of the Institute for Systems Theory and
Automatic Control and professor at the University of
Stuttgart. His research interests include networked
control, cooperative control, predictive control, and
nonlinear control with application to a wide range of
fields including systems biology. For the years 2017-2020 Frank 
served as President of the International Federation of Automatic
Control (IFAC) and for the years 2012-2020 as Vice President of 
the German Research Foundation DFG.

\end{IEEEbiography}


% \begin{IEEEbiographynophoto}{Patricia Pauli}

\vfill

\end{document}